\documentclass[runningheads]{llncs}

\usepackage{eccv}
\usepackage{prettyref}
\usepackage{mathrsfs}
\usepackage{graphicx}
\usepackage{wrapfig}
\usepackage{subcaption}
\usepackage{bbold}
\usepackage{tabu}
\usepackage{multirow}
\usepackage{booktabs}
\usepackage{tabularx}
\usepackage{booktabs}
\usepackage{pifont}  % for \checkmark
\usepackage{amsmath,amsfonts,amssymb}
\usepackage{algorithm}

\usepackage[table]{xcolor}
\usepackage[noend]{algpseudocode}
\usepackage[colorlinks,allcolors=gray,hypertexnames=true]{hyperref} 
\newrefformat{fig}{Figure~\ref{#1}}
\newrefformat{par}{Section~\ref{#1}}
\newrefformat{appen}{Appendix~\ref{#1}}
\newrefformat{sec}{Section~\ref{#1}}
\newrefformat{sub}{Section~\ref{#1}}
\newrefformat{table}{Table~\ref{#1}}
\newrefformat{ass}{Assumption~\ref{#1}}
\newrefformat{alg}{Algorithm~\ref{#1}}
\newrefformat{def}{Definition~\ref{#1}}
\newrefformat{thm}{Theorem~\ref{#1}}
\newrefformat{cor}{Corollary~\ref{#1}}
\newrefformat{lem}{Lemma~\ref{#1}}
\newrefformat{step}{Step~\ref{#1}}
\newrefformat{ln}{Line~\ref{#1}}
\newrefformat{rem}{Remark~\ref{#1}}
\newrefformat{eq}{Equation~\ref{#1}}
\newrefformat{pb}{Problem~\ref{#1}}
\newrefformat{it}{Item~\ref{#1}}
\newrefformat{te}{Term~\ref{#1}}
\def\Eqref Eq:#1:{\eqref{eq:#1}}
\newrefformat{Eq}{Equation~\Eqref#1:}

\newcommand{\argminH}{\mathbb{argmin}}

\newcommand{\ours}{{PoseShield}\xspace}
\definecolor{myred}{RGB}{179,64,64}
\definecolor{mygreen}{RGB}{64,179,64}

\usepackage{eccvabbrv}

\usepackage{graphicx}
\usepackage{booktabs}
\newtheorem{assumption}{Assumption}

\usepackage[accsupp]{axessibility}  % Improves PDF readability for those with disabilities.

\usepackage{hyperref}

\usepackage{orcidlink}

\begin{document}

% ---------------------------------------------------------------
\title{PoseShield: Neural Collision Fields for Human Self-Collision Resolution} 

\titlerunning{PoseShield}

\author{Zhengyuan Li\inst{1}\orcidlink{0009-0000-6376-2780} \and
Zeyun Deng\inst{1} \and
Yifan Shen\inst{3} \and Liangyan Gui\inst{3}\orcidlink{0009-0005-8204-3577} \and Miaolan Xie\inst{1}\orcidlink{0000-0001-8511-9649} \and Joseph Campbell\inst{1}\orcidlink{0000-0002-7924-8548} \and Xifeng Gao\inst{2}\orcidlink{0000-0003-0829-7075} \and Kui Wu\inst{2}\orcidlink{0000-0003-3326-7943} \and Zherong Pan\inst{2}\orcidlink{0000-0001-9348-526X} \and Aniket Bera\inst{1}\orcidlink{0000-0002-0182-6985}}

\authorrunning{Z.~Li et al.}

\institute{Purdue University, West Lafayette IN 47907, USA \and LightSpeed Studios, Bellevue WA 98004, USA \and
University of Illinois Urbana-Champaign, Champaign IL 61820, USA}
\maketitle
\begin{abstract}
Self-collision remains a persistent challenge in SMPL-based human pose estimation and motion generation. Under extreme articulations or stochastic motion synthesis, generated meshes frequently exhibit self-penetrations, leading to physically implausible results. We propose {\ours}, a neural collision constraint defined directly in SMPL pose space. We formulate collision correction as a constrained optimization problem and connect the learned constraint with the Eikonal equation. Enforcing Eikonal regularization ensures non-vanishing gradients near the collision boundary, improving numerical stability and robustness of the optimization process. Unlike prior methods that operate in the mesh space or rely on heuristic penalties, our approach operates directly in the low-dimensional space of human poses and is theoretically grounded. The same learned constraint extends to human motion sequences, providing a generator-agnostic post-hoc collision corrector without retraining the underlying motion model. Experiments on a newly constructed SMPL pose benchmark show that our method achieves a 95.8\% success rate and outperforms state-of-the-art baselines.
\end{abstract}    
\section{Introduction}
\label{sec:intro}

% background
Parametric human body models such as SMPL~\cite{loper2023smpl}, SMPL+H~\cite{romero2022embodied}, and SMPL-X~\cite{SMPL-X} have become the standard geometric representation in human pose estimation~\cite{kanazawa2018end,bogo2016keep,choutas2022accurate,kocabas2020vibe,zhang2020object,feng2024chatpose,zhang2024rohm} and motion generation~\cite{zhang2024large,hymotion2025,li2023object,lin2023motion}. Thanks to their explicit mesh topology and low-dimensional pose parameterization, these models enable efficient optimization and tight integration with learning-based pipelines. Despite their widespread adoption, \emph{self-collision} remains a persistent challenge.

%problem

Self-collisions arise across diverse SMPL-based applications, ranging from motion-capture-based reconstruction to motion synthesis. For reconstruction, COAP~\cite{mihajlovic2022coap} shows that poses from datasets such as PROX~\cite{PROX:2019} may contain body self-intersections.
For motion synthesis, recent self-collision-aware generation work~\cite{herrmann2025self} further analyzes that modern motion synthesis methods~\cite{tevethuman,guo2024momask} are prone to producing motions with non-negligible self-intersections under stochastic sampling. Such artifacts degrade physical plausibility. Consequently, there is a pressing need for a \emph{reliable, post-hoc self-collision resolution approach}. An ideal approach must decouple collision handling from specific generative priors to act as a universal refinement module, ensuring geometric consistency across diverse SMPL-based scenarios without relying on original generative conditions like textual prompts or reference images.

Classical collision-handling methods have been extensively studied in geometry processing and physical simulation~\cite{harmon2009asynchronous,Li2021CIPC, sassen2024repulsive}. These approaches typically operate directly in \emph{mesh space}, optimizing over vertex positions and leveraging penalty energies or interior-point formulations. While effective in simulation settings where a collision-free reference configuration is available, these methods are not naturally compatible with learning-based SMPL pipelines. In human pose estimation and motion generation, the optimization variable is the \emph{pose parameter} $\boldsymbol{\theta}$ rather than raw mesh vertices. Moreover, a collision-free reference pose is generally unavailable. As a result, classical mesh-space solvers cannot be directly transferred to pose-space constrained optimization.
\begin{figure}[t]
\centering  
\includegraphics[width=0.6\columnwidth]{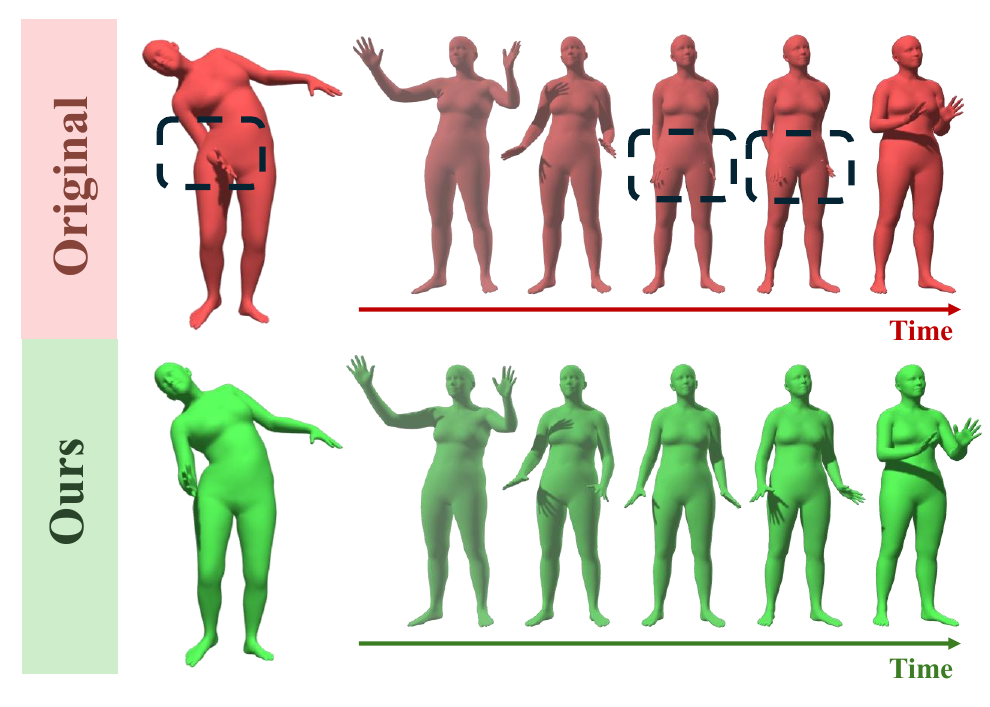}

\caption{\ours provides a robust collision constraint for SMPL-based bodies. While effective for individual human poses, it naturally extends to temporally consistent motion generation.}
\vspace{-6mm}
\label{fig:teaser}
\end{figure}
In the domain of human body modeling, several works incorporate interpenetration penalties as ``soft constraints'' during model fitting or as auxiliary losses in learning-based predictors~\cite{SMPL-X,tzionas2016capturing,guan2009estimating,bogo2016keep,herrmann2025self}. Recent volumetric human modeling approaches~\cite{mihajlovic2022coap,mihajlovic2025volumetricsmpl} learn smooth articulated occupancy representations of posed human bodies. Such models provide continuous geometric fields that can implicitly encode collisions through occupancy evaluation. Other recent learning-based approaches~\cite{tan2022repulsive,tan2022n} train smooth neural collision classifiers and use them as differentiable constraints in post-hoc optimization. Unlike these methods, our approach is motivated by the regularity conditions required by gradient-based constrained optimization algorithms.

In this work, we formulate collision resolution as a constrained optimization problem that searches for the nearest collision-free pose to a self-intersecting configuration. Under this formulation, we introduce \textbf{\ours}, a neural collision constraint function defined directly in SMPL pose space, and solve the resulting problem using well-established gradient-based constrained optimization algorithms such as SLSQP and augmented Lagrangian methods~\cite{nocedal2006numerical}. For \ours to work reliably within these solvers, two requirements must be met. First, its sign must reliably indicate collision status: positive for collision-free poses and negative for self-intersecting ones. Second, these algorithms require the Linear Independence Constraint Qualification (LICQ)~\cite{nocedal2006numerical}, which demands a nonvanishing constraint gradient away from the constraint boundary. Our key insight is the connection between this regularity requirement and the \emph{Eikonal equation}: enforcing an Eikonal regularization on \ours ensures that its gradient norm is bounded away from zero throughout pose space, thereby satisfying LICQ and improving numerical stability. Such a constraint function is guaranteed to exist: the signed distance function to the collision boundary, characterized as the unique viscosity solution of the Eikonal equation~\cite{crandall1992user}. Together, these properties yield a principled, self-contained post-hoc collision resolver for individual SMPL poses. We further extend our framework to human motion generation, enabling post-hoc self-collision correction for human motion sequences without knowing the generator. In summary, our contributions are:
\begin{itemize}
\item We propose a principled theoretical framework for self-collision resolution in SMPL pose space by formulating it as a constrained optimization problem with a learnable differentiable collision constraint. We further show that, under suitable assumptions on the learned collision field, gradient-based constrained solvers admit global and local convergence guarantees for this problem.

\item We introduce \textbf{\ours}, a neural collision constraint trained with Eikonal regularization in pose space, and show how its training objective is designed to make these theoretical assumptions approximately hold in practice. In particular, we prove that the Eikonal training loss bounds the volume of pose-space regions where LICQ fails to hold, providing a quantitative connection between training accuracy and solver reliability.

\item We demonstrate that the learned collision constraint can be seamlessly reused for temporally consistent motion correction without retraining the underlying motion generator, providing \textbf{a generator-agnostic post-processing module for human motion synthesis}.

\item We perform evaluations on the newly curated \textbf{Human with Collisions} (HwC) dataset and the PROX dataset, where \ours substantially outperforms prior post-hoc collision-handling baselines.
\end{itemize}
\section{Related Work}
\paragraph{Human Self-Collision in SMPL-based Modeling.}
Self-collision is a long-standing issue for parametric human models like SMPL/SMPL-X, particularly under extreme articulations or depth ambiguities in monocular reconstruction. Early approaches primarily relied on coarse geometric heuristics, approximating the human body with primitive proxies such as capsules~\cite{bogo2016keep} to simplify intersection queries. Subsequently, the field shifted toward mesh-level penetration penalties. For example, SMPLify-X~\cite{SMPL-X} adapted differentiable self-interpenetration terms for expressive body capture, often implemented via distance-field-style losses~\cite{tzionas2016capturing,ballan2012motion} and accelerated intersection detection~\cite{Karras:2012:MPC:2383795.2383801}. This paradigm has been extended to scene-aware constraints in PROX~\cite{PROX:2019}. For motion generation, recent work~\cite{herrmann2025self} employs efficient sphere-based proxies to incorporate self-intersection losses during the training of human motion models~\cite{tevethuman,guo2024momask}. Recently, implicit representations like COAP~\cite{mihajlovic2022coap} and VolumetricSMPL~\cite{mihajlovic2025volumetricsmpl} have advanced the state-of-the-art by demonstrating that learned occupancy fields can effectively reduce self-intersections through gradient-based refinement. While effective in reducing interpenetration for interaction modeling, these approaches are not explicitly designed to provide regularity guarantees for post-hoc constrained optimization. Relevantly, implicit neural representations have also been explored for modeling geometric or interaction constraints in the pose space~\cite{kulkarni2024nifty}. Our work instead focuses on learning a collision constraint with properties tailored to support stable and theoretically grounded pose-space optimization.
\paragraph{General Mesh Collision Resolution.}
Classical collision-handling techniques, including penalty-based energies~\cite{fisher2001deformed,chen2023shortest,chen2025offset}, interior-point formulations~\cite{harmon2009asynchronous,Li2021CIPC,fang2021guaranteed,huang2025intersection}, and geometric frameworks like Repulsive Shells~\cite{sassen2024repulsive}, share a fundamental limitation: they operate directly in mesh space by optimizing raw vertex positions. This makes them neither directly applicable nor computationally practical in pose-space optimization. Furthermore, these methods often necessitate a collision-free reference configuration, which is generally unavailable in modern pose-prediction pipelines~\cite{tan2021lcollision,tan2022n,tevethuman,zhang2024rohm}. While specialized neural architectures have addressed collisions in specific domains, such as ContourCraft~\cite{grigorev2024contourcraft} and Self-Supervised Collision Handling~\cite{santesteban2021self} for garments, or Quaffure~\cite{stuyck2025quaffure} for hair simulation, these solutions are tailored to their respective generative priors and do not generalize to the high-dimensional articulated manifold of human bodies. Ultimately, existing neural classifiers used for post-hoc resolution~\cite{tan2021lcollision,tan2022n,zesch2023neural} often learn decision boundaries or proxy penetration depth signals for gradient guidance, but do not enforce global regularity conditions needed for robust constrained optimization.
% provide only a binary or heuristic signal that lacks the necessary continuity and theoretical convergence guarantees for robust constrained optimization.
% Such a lack of mathematical structure often leads to numerical instability or failure to find a feasible solution in complex self-penetrating scenarios.

\paragraph{Neural Solution of the Eikonal Equation.}
Solving the Eikonal equation~\cite{sethian1996fast,zhao2005fast} subject to a sign constraint in 3D space corresponds to computing the Signed Distance Field (SDF) of an arbitrary shape. Thanks to its expressive power, the neural approximation of SDFs has become central to modern 3D generative models~\cite{park2019deepsdf,li2023diffusion,yariv2024mosaic}. Although the SDF can represent highly complex geometry, these methods all assume a low-dimensional underlying domain—namely, the standard 3D Euclidean space. More recently, some research~\cite{ni2023ntfields} demonstrated configuration-space distance fields whose gradients guide collision-free planning. Closer to our setting, PoseNDF~\cite{tiwari22posendf} and NRDF~\cite{he2024nrdf} also learn SDF-style neural fields on the articulated pose manifold $SO(3)^K$, using their zero level-set to encode the data-driven manifold of \emph{plausible} poses and serving as generic priors for tasks such as denoising and inverse kinematics. In contrast, we model a specific geometric constraint, self-collision, and design the field for constrained pose-space optimization.
\section{Problem: SMPL Self-Collision Resolution}
\label{sec:problem}
In this section, we define the collision-resolution problem for the SMPL family of parametric human body models~\cite{loper2023smpl,SMPL-X}. 
An SMPL mesh is defined by shape parameters $\boldsymbol{\beta} \in \mathbb{R}^{d_\beta}$ and pose parameters $\boldsymbol{\theta} \in \mathbb{R}^{d_\theta}$. 
Given $(\boldsymbol{\beta}, \boldsymbol{\theta})$, the SMPL function produces a mesh:
\[
X = \mathcal{M}(\boldsymbol{\beta}, \boldsymbol{\theta}),
\]
where the mesh connectivity $\mathcal{T}$ is predefined and independent of
$\boldsymbol{\beta}$ and $\boldsymbol{\theta}$.
Due to imperfect motion capture or errors introduced by neural motion predictors~\cite{tevethuman}, the generated SMPL meshes may exhibit self-collisions. In many practical scenarios, such as motion correction, the body shape remains fixed across frames. Therefore, we assume the shape parameter $\boldsymbol{\beta}$ is fixed and only optimize the pose parameter $\boldsymbol{\theta}$. We also ignore the global translation and rotation, since self-collision is invariant to them. To characterize whether a posed SMPL mesh is collision-free, we employ an exact mesh self-intersection test.
Concretely, given a pose $\boldsymbol{\theta}$, we first decode it into a mesh $X=\mathcal{M}(\boldsymbol{\beta},\boldsymbol{\theta})$ and then apply a classical collision detector (e.g., FCL~\cite{pan2012fcl}) to obtain a binary collision indicator defined as:
\[
\iota\big(\mathcal{M}(\boldsymbol{\beta},\boldsymbol{\theta})\big)\in\{-1,+1\},
\]
which equals $-1$ if the mesh exhibits self-penetration and $+1$ otherwise. Given a self-colliding SMPL configuration $(\boldsymbol{\beta},\boldsymbol{\theta}_0)$, our goal is to find a corrected pose $\boldsymbol{\theta}$ such that the decoded mesh:
\[
X=\mathcal{M}(\boldsymbol{\beta},\boldsymbol{\theta}),
\]
is collision-free while remaining visually close to the original configuration. Let $d_{\mathrm{SMPL}}(\cdot,\cdot)$ denote a distance function that measures the discrepancy between two SMPL poses of a fixed shape. We formulate SMPL self-collision resolution as the following constrained optimization problem:
\begin{equation}
\label{eq:smpl_collisionresolution}
\boldsymbol{\theta}^\star
=
\arg\min_{\boldsymbol{\theta}}
d_{\mathrm{SMPL}}(\boldsymbol{\theta},\boldsymbol{\theta}_0)
\quad
\text{subject to}
\quad
\iota\big(\mathcal{M}(\boldsymbol{\beta},\boldsymbol{\theta})\big)=+1.
\end{equation}
\section{Learning a Differentiable Collision Constraint}
As discussed in \cref{sec:problem}, the exact collision indicator $\iota(\cdot)$ provides a binary feasibility test. However, $\iota(\cdot)$ is non-differentiable with respect to the pose parameter $\boldsymbol{\theta}$, making gradient-based constrained optimization intractable. Modern constrained optimization algorithms, including Sequential Least Squares Programming (SLSQP), modified differential multiplier methods, and augmented Lagrangian approaches, require at least $C^1$ smooth constraint functions in order to compute well-defined gradients. Constraint functions must satisfy certain constraint qualifications such as LICQ for the iterates to reliably converge to a KKT point~\cite{bertsekas1997nonlinear}. Since $\iota(\cdot)$ yields only binary feasibility labels and is discontinuous with zero gradient almost everywhere, these methods cannot be applied.

To enable robust constrained optimization in pose space, we introduce a surrogate, learnable neural constraint function $g(\boldsymbol{\theta})$, whose superlevel set 
$\{\boldsymbol{\theta} \mid g(\boldsymbol{\theta}) \ge 0\}$ 
approximates the set of collision-free SMPL poses under a fixed shape $\boldsymbol{\beta}$. With the surrogate function and given a self-colliding initial pose $\boldsymbol{\theta}_0$, the optimization problem becomes:
\begin{equation}
\label{eq:smpl_collisionresolution_smooth_main}
\boldsymbol{\theta}^\star
=
\arg\min_{\boldsymbol{\theta}}
d_{\mathrm{SMPL}}(\boldsymbol{\theta},\boldsymbol{\theta}_0)
\quad
\text{subject to}
\quad
g(\boldsymbol{\theta}) \ge C_l .
\end{equation}

where $C_l$ is a margin threshold. Theoretically, $C_l = 0$ is the optimal choice assuming a perfectly learned constraint $g$. In practice, however (as detailed in \cref{sec:ablation}), we demonstrate that adjusting this threshold facilitates a controllable trade-off between maintaining geometric fidelity to the input pose and the effectiveness of collision resolution.

Our method is illustrated in \cref{fig:pipeline}. The neural collision-handling process operates in a latent space and follows a standard constrained optimization formulation, with $g$ serving as a differentiable constraint function. We first analyze, in \cref{sec:theory}, the convergence behavior of standard gradient-based constrained solvers on this formulation under three assumptions on $g$. We then introduce in \cref{sec:eikonal} our Eikonal regularization, which is designed to encourage the most non-trivial of these assumptions. A complete theoretical analysis is in the supplement. Building on these theoretical insights, \cref{sec:training} and \cref{sec:td} detail the practical training strategies for learning the differentiable collision constraint $g$. Finally, in \cref{sec:motion}, we generalize our framework from static pose optimization to accommodate continuous, temporally consistent human motion sequences.

\begin{figure*}[t] % or [htbp]
\centering
\includegraphics[width=0.9\linewidth]{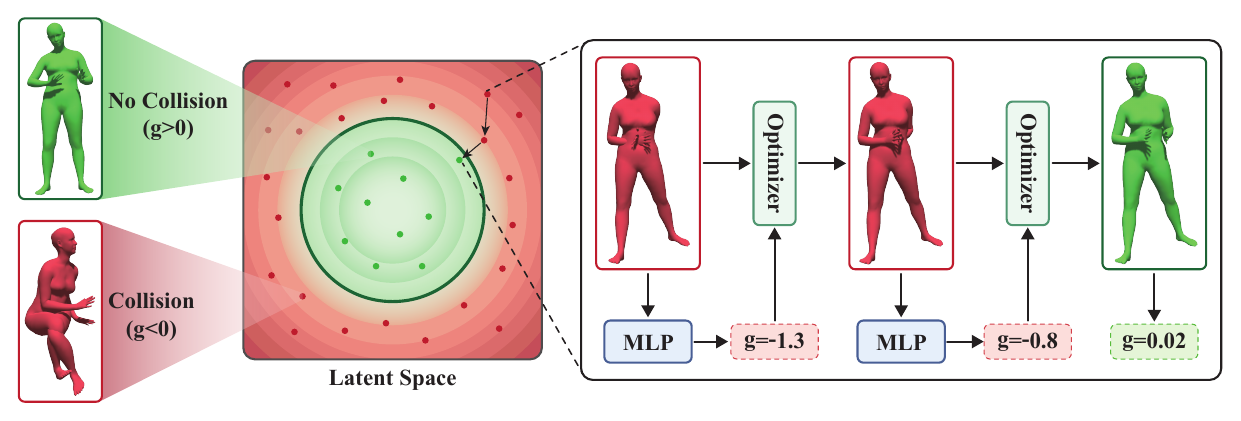}
\caption{\textbf{Neural collision handling with \ours}. \ours approximates the signed distance function (SDF) to the boundary between colliding and collision-free regions in the latent space. In practice, collision resolution algorithms optimize a self-penetrating sample toward its nearest point in the collision-free region, where the learned neural field provides local gradient guidance throughout the optimization process.}
\label{fig:pipeline}
\end{figure*}

\subsection{Convergence Analysis}
\label{sec:theory}

We analyze the convergence behavior of standard gradient-based constrained solvers on the surrogate problem in~\eqref{eq:smpl_collisionresolution_smooth_main}. Let $\Omega \subset [-1,1]^{d_\theta}$ denote a bounded region of plausible SMPL poses in the 6D rotation representation. For the analysis, we set the threshold $C_l = 0$ and adopt the squared-Euclidean distance $d_{\mathrm{SMPL}}(\boldsymbol{\theta}, \boldsymbol{\theta}_0) = \|\boldsymbol{\theta} - \boldsymbol{\theta}_0\|^2$. We further assume the minimizer is interior to $\Omega$ and non-degenerate, so that $g \geq 0$ is the only active constraint.

Classical nonlinear optimization theory~\cite{nocedal2006numerical} establishes strong convergence guarantees for gradient-based constrained solvers under appropriate regularity conditions on the constraint function. We state three assumptions on the learned field $g$ that together suffice for these guarantees in our setting.

\begin{assumption}[Smoothness]\label{asmp:smoothness}
$g$ is $C^2$ on $\Omega$ with Lipschitz-continuous gradient and Hessian.
\end{assumption}

\begin{assumption}[Feasibility consistency]\label{asmp:feasibility}
The sign of $g$ correctly identifies the collision status: $g(\boldsymbol{\theta}) \geq 0$ if and only if $\boldsymbol{\theta}$ is a collision-free pose.
\end{assumption}

\begin{assumption}[Approximate Eikonal property]\label{asmp:approx_eikonal}
There exists $\delta \in [0,1)$ such that $1 - \delta \leq \|\nabla_{\boldsymbol{\theta}} g(\boldsymbol{\theta})\|_2 \leq 1 + \delta$ for all $\boldsymbol{\theta} \in \Omega$.
\end{assumption}

The first assumption ensures that gradient-based methods are well-defined. The second ensures that the learned constraint faithfully represents the collision-free set. The third makes $g$ behave like an approximate signed distance function in pose space; in particular, $\|\nabla g\| \geq 1 - \delta > 0$ implies that the Linear Independence Constraint Qualification (LICQ) holds globally. Under these three assumptions, standard gradient-based constrained solvers (e.g., SLSQP) admit both global and local convergence guarantees:

\begin{theorem}[Global Convergence and Complexity]\label{thm:global}
Consider problem~\eqref{eq:smpl_collisionresolution_smooth_main} under Assumptions~\ref{asmp:smoothness},~\ref{asmp:feasibility}, and~\ref{asmp:approx_eikonal}. Then:
\begin{enumerate}
\item[\textup{(i)}] \textbf{Global LICQ:} The constraint qualification holds globally on $\Omega$.
\item[\textup{(ii)}] \textbf{Global Convergence:} From any starting pose $\boldsymbol{\theta}_{0} \in \Omega$, a standard line-search SQP method with an $\ell_1$ merit function (and sufficiently large penalty parameter) produces iterates whose every accumulation point is a first-order KKT point $(\boldsymbol{\theta}^\star, \lambda^\star)$.
\item[\textup{(iii)}] \textbf{Iteration Complexity:} An $\varepsilon$-approximate KKT point---satisfying
$$\bigl\| 2(\boldsymbol{\theta}_k - \boldsymbol{\theta}_0) - \lambda_k \nabla_{\boldsymbol{\theta}} g(\boldsymbol{\theta}_k) \bigr\| \leq \varepsilon, \quad |\min(0,\, g(\boldsymbol{\theta}_k))| \leq \varepsilon, \quad \lambda_k \geq 0, \quad |\lambda_k\, g(\boldsymbol{\theta}_k)| \leq \varepsilon,$$
is no harder to obtain than in unconstrained smooth optimization, whose worst-case first-order complexity is $\mathcal{O}(\varepsilon^{-2})$.
\end{enumerate}
\end{theorem}

\begin{theorem}[Local Convergence]\label{thm:local}
Consider problem~\eqref{eq:smpl_collisionresolution_smooth_main} under Assumptions~\ref{asmp:smoothness},~\ref{asmp:feasibility}, and~\ref{asmp:approx_eikonal}. Let $\boldsymbol{\theta}^\star$ be a local minimizer, and assume the initial pose is infeasible, i.e., $g(\boldsymbol{\theta}_0) < 0$. Then:
\begin{enumerate}
\item[\textup{(i)}] \textbf{LICQ and Strict Complementarity:} LICQ holds at $\boldsymbol{\theta}^\star$, and the unique KKT multiplier satisfies $0 < \frac{2}{1+\delta}\|\boldsymbol{\theta}^\star - \boldsymbol{\theta}_0\| \leq \lambda^\star \leq \frac{2}{1-\delta}\|\boldsymbol{\theta}^\star - \boldsymbol{\theta}_0\|$.
\item[\textup{(ii)}] \textbf{SOSC and Fast Convergence:} Define $\kappa \triangleq \lambda^\star\|\nabla_{\boldsymbol{\theta}}^2 g(\boldsymbol{\theta}^\star)\|_2$. If $\kappa < 2$, then the full Lagrangian Hessian is positive definite (implying Second-Order Sufficient Conditions), and SQP with exact Hessian converges locally to $(\boldsymbol{\theta}^\star, \lambda^\star)$ at a quadratic rate. A quasi-Newton (BFGS) variant satisfying the Dennis--Mor\'e condition converges superlinearly.
\end{enumerate}
\end{theorem}

\begin{proof}
The complete proofs are provided in the supplementary material.
\end{proof}

\subsection{Eikonal Regularization}
\label{sec:eikonal}

Among the three assumptions in~\cref{sec:theory}, smoothness is automatically satisfied by an MLP with smooth activations (e.g., Softplus), and feasibility consistency is encouraged by direct sign supervision (\cref{sec:training}). The approximate Eikonal property (Assumption~\ref{asmp:approx_eikonal}), however, requires a dedicated regularizer. The ideal pointwise target is
\begin{equation}
\label{eq:eikonal_pose}
\|\nabla_{\boldsymbol{\theta}} g(\boldsymbol{\theta})\| = 1,
\end{equation}
which we encourage in expectation over $\Omega$ (assuming a uniform probability measure for analysis) by minimizing the average violation:
\begin{equation}
\begin{aligned}
\label{eq:eikonal}
\mathcal{L}_\text{grad}^i &= \big|\|\nabla g(\boldsymbol{\theta}_i)\|-1\big|,\\
\mathcal{L}_{\mathrm{grad}}
&=
\mathbb{E}_{\boldsymbol{\theta}\sim\Omega}
\Big[
\big|\|\nabla g(\boldsymbol{\theta})\|-1\big|
\Big].
% \mathcal{L}_\text{grad} &= \frac{1}{N}\sum_{i=1}^N \big|\|\nabla g(\boldsymbol{\theta}_i)\|-1\big|,
\end{aligned}
\end{equation}

The following proposition shows that this loss provides a quantitative volume bound on the regions where Assumption~\ref{asmp:approx_eikonal} fails:

\begin{proposition}[Volume Bound on Approximate Eikonal Failure]\label{prop:approx_eikonal}
Let $S_\delta$ denote the region where the approximate Eikonal condition fails for a given margin $\delta \in (0,1)$:
$$S_\delta = \left\{ \boldsymbol{\theta} \in \Omega \;\middle|\; \big|\|\nabla_{\boldsymbol{\theta}} g(\boldsymbol{\theta})\| - 1\big| > \delta \right\}.$$
If the expected Eikonal loss satisfies $\mathcal{L}_{\mathrm{grad}} \le \varepsilon$, then the probability measure of this failure region is bounded by:
\begin{equation}\label{eq:failure_bound}
\mathbb{P}(\boldsymbol{\theta} \in S_\delta) \le \frac{\varepsilon}{\delta}.
\end{equation}
\end{proposition}
\begin{proof}
The complete proofs are provided in the supplementary material.
\end{proof}

Combined with the fact that the approximate Eikonal property implies LICQ ($\|\nabla g\| \geq 1 - \delta > 0$), this proposition provides a quantitative link between training accuracy and the volume of pose-space regions where LICQ is guaranteed to hold, directly supporting the convergence guarantees in~\cref{sec:theory}.

\subsection{Training Objective}
\label{sec:training}

To enforce Assumption~\ref{asmp:feasibility}, we impose boundary supervision using the exact collision indicator
$\iota$.
Concretely, we encourage:
\[
g(\boldsymbol{\theta}) \, \iota(\mathcal{M}(\boldsymbol{\beta},\boldsymbol{\theta})) > 0,
\]
so that $g(\boldsymbol{\theta})>0$ for collision-free poses and $g(\boldsymbol{\theta})<0$ otherwise.
Together, Eikonal regularization and sign supervision encourage $g(\boldsymbol{\theta})$ to approximate a SDF-like function to the collision boundary in SMPL pose space. Using Monte Carlo sampling, we construct:
$\mathcal{D}_{\theta}=\{\langle \boldsymbol{\theta}_i, \iota_i\rangle\}_{i=1}^{N}$,
where $\iota_i=\iota(\mathcal{M}(\boldsymbol{\beta},\boldsymbol{\theta}_i))$,
and optimize the empirical PINNs-style~\cite{raissi2019physics} objective:
\begin{equation}
\begin{aligned}
\label{eq:vanilla-PINNs-theta}
\mathcal{L}_\text{sign}^i &= -\min\big(g(\boldsymbol{\theta}_i)\iota_i, 0\big),\\
\mathcal{L}_\text{Eikonal} &= \frac{1}{|D_\theta|}\sum_{i=1}^{|D_\theta|}
\left(\mathcal{L}_\text{grad}^i + \mathcal{L}_\text{sign}^i\right).
\end{aligned}
\end{equation}
This objective is a direct high-dimensional analogue of standard Eikonal training used in low-dimensional SDF learning~\cite{park2019deepsdf,ni2021robust}.

\subsection{Temporal-Difference Variant}
\label{sec:td}
Inspired by the connection between $\mathcal{L}_\text{grad}$ and the Temporal Difference (TD) loss in offline reinforcement learning, and following Ni et al.~\cite{ni2025physicsinformed} who show that a finite-timestep TD variant improves training, we replace $\mathcal{L}_\text{grad}$ with a symmetric TD loss parameterized by a small timestep $\Delta t$:
\begin{align} 
\label{eq:finite-diff}
\mathcal{L}_\text{TD}^i=\big|g(\boldsymbol{\theta}_i+\boldsymbol{v}_i\Delta t)-g(\boldsymbol{\theta}_i-\boldsymbol{v}_i\Delta t)-2\Delta t\big|, 
\end{align} 
where $\boldsymbol{v}_i=\nabla_{\boldsymbol{\theta}} g(\boldsymbol{\theta}_i)/\|\nabla_{\boldsymbol{\theta}} g(\boldsymbol{\theta}_i)\|$ is the normalized velocity function. We find that replacing $\mathcal{L}_\text{grad}$ with $\mathcal{L}_\text{TD}$ improves the performance as demonstrated in~\cref{tab:main_result}.

\subsection{Collision Resolution for Motion Sequence}
\label{sec:motion}

Our formulation naturally extends from static pose optimization to human motion sequences. Specifically, the sampling process in human motion diffusion models~\cite{tevethuman,meng2025rethinking,zhang2023remodiffuse,zhong2024smoodi,xu2023interdiff,dai2024motionlcm,ruiz2025mixermdm,hong2025salad,li2023object,li2025simmotionedit} and human motion flow matching models~\cite{hymotion2025} can be formulated as a differentiable generative mapping: $\mathbf{m} = f(\mathbf{x})$, where $\mathbf{x}$ denotes the input noise, and $\mathbf{m}$ is the generated human motion sequence. To enforce task-specific constraints (e.g., obstacle or self-collision avoidance), DNO~\cite{karunratanakul2024optimizing} proposes a conditional motion synthesis formulated as:
\begin{equation}
\mathbf{x}^\star = \arg\min_{\mathbf{x}} \mathcal{Q }\big(f(\mathbf{x})\big),
\end{equation}
where $\mathcal{Q}$ is a user-defined criterion measuring constraint satisfaction or motion plausibility. This formulation is naturally compatible with our constrained optimization formulation in \prettyref{eq:smpl_collisionresolution_smooth_main}. Specifically, we first train \ours as usual for static human poses. Suppose we are given a motion sequence with self-collisions defined as $\mathbf{m}_s = [\boldsymbol{\theta}_s^0, \boldsymbol{\theta}_s^1, \cdots, \boldsymbol{\theta}_s^T]$, and let the optimization variable be $\mathbf{m} = [\boldsymbol{\theta}^0, \boldsymbol{\theta}^1, \cdots, \boldsymbol{\theta}^T]$, we define the objective function $\mathcal{Q}$ as follows:
\begin{equation}
\mathcal{Q}(\mathbf{m}) = \sum_{i=0}^{T} \max\big(C_l-g(\boldsymbol{\theta}^i), 0 \big) + \lambda_m \, d_{motion}(\mathbf{m}, \mathbf{m}_s),
\end{equation}
where the first term penalizes violations of the collision constraint, and the second term enforces proximity to the original motion. The coefficient $\lambda_m$ balances collision resolution and motion preservation.  The distance metric $d_{motion}$ is detailed in the supplementary material. Note that we avoid the hard constraints from~\prettyref{eq:smpl_collisionresolution_smooth_main}, which are disallowed by DNO, and replace them with soft constraint functions. 
% $\max\big(C_l-g(\boldsymbol{\theta}_s^i), 0 \big)$.
\section{Evaluation}
In this section, we evaluate the performance of \ours and compare it against baselines. We first discuss collision resolution for static human poses (\prettyref{sec:human_pose}), and then ablate key design choices and demonstrate properties of our method (\prettyref{sec:ablation}). Next, we scale the constraint function learned from individual poses to human motion sequences (\prettyref{sec:motion_seq}). 

\subsection{\label{sec:human_pose}Application: Static Human Pose}
\paragraph{Humans with Collisions (HwC) Dataset.}
Despite the availability of various human pose datasets~\cite{mahmood2019amass,delmas2024posescript,ionescu2013human3}, meshes sampled solely from the collision-free ground-truth distribution do not expose \ours to any self-colliding examples. To address this, we introduce the \textbf{Humans with Collisions (HwC)} dataset of nearly $931k$ SMPL poses, which is detailed in the supplementary. A subset of $500$ representative self-penetrating meshes serves for benchmarking. This dataset also serves as the sampling pool to approximate $\Omega$ in \cref{sec:eikonal}. In addition, we follow previous work~\cite{mihajlovic2022coap} and evaluate the methods on a subset of PROX dataset~\cite{PROX:2019}.

\paragraph{Metrics.} Following~\cite{tan2022n}, we adopt the following metrics: 
\begin{itemize}
\item Success Rate (SCC): The rate of the collision resolution method producing penetration-free meshes.
\item Penetration Depth Reduction (PDR): The ratio of reduced penetration depth (PD)~\cite{pan2012fcl} relative to the original penetration depth defined as:
\begin{align*}
\text{PDR} = \max\left[1 - 
\frac{\text{PD after optimization}}{\text{PD before optimization}}, 0\right].
\end{align*}
\item Mean Vertex Distance (MVD): The average $L_2$ distance of the vertices between the original and the optimized mesh. (An ideal collision handler resolves collisions while keeping the resulting pose as close as possible to the original sample.)
\end{itemize}

\paragraph{Baselines.}
We consider the following baselines: 1) Torch-mesh-isect~\cite{tzionas2016capturing}: An open-source tool developed specifically for SMPL-based collision resolution. 2) Classifier-baseline: Inspired by N-Penetrate~\cite{tan2022n}, we replace the collision constraint function with the probability of being a non-colliding mesh predicted by a classifier using the cross-entropy loss. More details are deferred to the supplementary. 3) COAP~\cite{mihajlovic2022coap}: a volumetric occupancy-field approach that treats collision resolution as a sampling-based occupancy penalty in the 3D workspace rather than a direct pose-space constraint. 4) VolumetricSMPL~\cite{mihajlovic2025volumetricsmpl}: a follow-up volumetric body representation that extends COAP by modeling the human body as a signed distance field (SDF) instead of an occupancy field.

\paragraph{Implementation Details.}
For optimization, we use the standard SLSQP method implemented in SciPy~\cite{2020SciPy-NMeth}. The network is a 12-layer MLP with a hidden dimension of 512. We train it for 200 epochs, where we adopt active learning~\cite{tan2022n} to collect boundary samples every 40 epochs. The entire training process requires around 17 hours on a single GPU. For the loss term, the default $\Delta t$ in~\prettyref{eq:finite-diff} is 0.01. Inference takes $7.26$ seconds per pose on average.

\begin{table}[t]
\centering
\caption{Quantitative results and ablation study on the pose datasets. Our method achieves significant improvement in collision resolution ability over all baselines. $\uparrow$: higher values are better; $\downarrow$: lower values are better. 
Bold and underlined values indicate the \textbf{best} and \underline{second-best} performers in each category, respectively.
$\mathcal{L}_{\mathrm{grad}}$ and $\mathcal{L}_{\mathrm{TD}}$ represent the gradient loss terms; 
WD denotes the weighted distance metric in pose space (see \cref{sec:ablation}).}
\label{tab:main_result}
\scalebox{1}{
\begin{tabular}{l lll | ccc | ccc}
\toprule
\multirow{2}{*}{Method} & \multirow{2}{*}{$\mathcal{L}_{\mathrm{grad}}$} & \multirow{2}{*}{$\mathcal{L}_{\mathrm{TD}}$} & \multirow{2}{*}{WD} & \multicolumn{3}{c|}{HwC} & \multicolumn{3}{c}{PROX~\cite{PROX:2019}} \\
\cmidrule(lr){5-7} \cmidrule(lr){8-10}
 & & & & SCC$\uparrow$ & PDR$\uparrow$ & MVD$\downarrow$ & SCC$\uparrow$ & PDR$\uparrow$ & MVD$\downarrow$ \\ 
\midrule
%\textit{Baselines} & & & & & & & & & \\
Torch-mesh-isect~\cite{tzionas2016capturing} &  &  &  & 0.100 & 0.357 & 0.041 & 0.110 & 0.291 &  \underline{0.012} \\
Classifier baseline &  &  &  & 0.056 & 0.081 & \textbf{0.002} & 0.170 & 0.204 & \textbf{0.006} \\
COAP~\cite{mihajlovic2022coap} &  &  &  & 0.446 & 0.832 & 0.106 & 0.560& 0.775 & 0.016 \\
VolumetricSMPL~\cite{mihajlovic2025volumetricsmpl} &  &  &  &0.250 & 0.541 & 0.068 & 0.333&0.699 & {0.013} \\
\midrule
\textbf{Ours} & & \checkmark & \checkmark & \underline{0.958} & \underline{0.982} & \underline{0.059} & \underline{0.800} & \underline{0.893} & 0.021 \\

\midrule
%\textit{Ablations} & & & & & & & & & \\
Ours (w/o WD) & & \checkmark & & \textbf{0.960} & \textbf{0.987} & 0.067 & \textbf{0.850}& \textbf{0.918} & 0.036\\
Ours ($\mathcal{L}_{\mathrm{grad}}$) & \checkmark & & \checkmark & 0.862 & 0.917 & 0.062 & 0.520 & 0.615 & 0.019 \\
Ours ($\mathcal{L}_{\mathrm{grad}}$+$\mathcal{L}_{\mathrm{TD}}$) & \checkmark & \checkmark & \checkmark & 0.870 & 0.922 & 0.062 & 0.670 & 0.743 & 0.022 \\
Ours (w/o grad term) & & & \checkmark & 0.068  & 0.081 & 0.458 & 0.050 & 0.106 & 0.431 \\
\bottomrule
\end{tabular}}

\scalebox{1}{
\begin{tabular}{l | ccc}
\toprule
Method & SCC$\uparrow$ & PDR$\uparrow$ & MVD$\downarrow$ \\
\midrule

COAP~\cite{mihajlovic2022coap}
& 0.446 & 0.832 & 0.106 \\

VolumetricSMPL~\cite{mihajlovic2025volumetricsmpl}
& 0.250 & 0.541 & 0.068 \\
\midrule

\textbf{Ours}
& \underline{0.958} & \underline{0.982} & \underline{0.059} \\

\bottomrule
\end{tabular}}

\end{table}

\paragraph{Quantitative Results.}
As shown in \cref{tab:main_result}, our method substantially outperforms all baselines in both success rate and penetration depth reduction. The classifier baseline, which also performs latent-space optimization, achieves a very low success rate, highlighting that our proposed neural field is a crucial component for effective collision handling within such frameworks. Similarly, Torch-mesh-isect~\cite{tzionas2016capturing} applies triangle-level penetration losses in pose optimization, but its local face-level loss formulation~\cite{SMPL-X} prevents it from resolving deep self-penetrations, often causing it to converge to local minima and fail to move the mesh sufficiently. Among the baselines, COAP~\cite{mihajlovic2022coap} attains the highest SCC and PDR, but its performance degrades in challenging cases with severe penetrations. Most notably, our method increases the success rate on the HwC dataset from $0.446$ to $0.958$ while achieving a significantly lower MVD, which is also reflected in \cref{fig:analysis_combined}. This indicates that our formulation resolves self-collisions 
with smaller pose deviations,
empirically approximating minimal-distance corrections to collision-free configurations.

\begin{figure}[t] % or [htbp]
\centering
\includegraphics[width=\columnwidth]{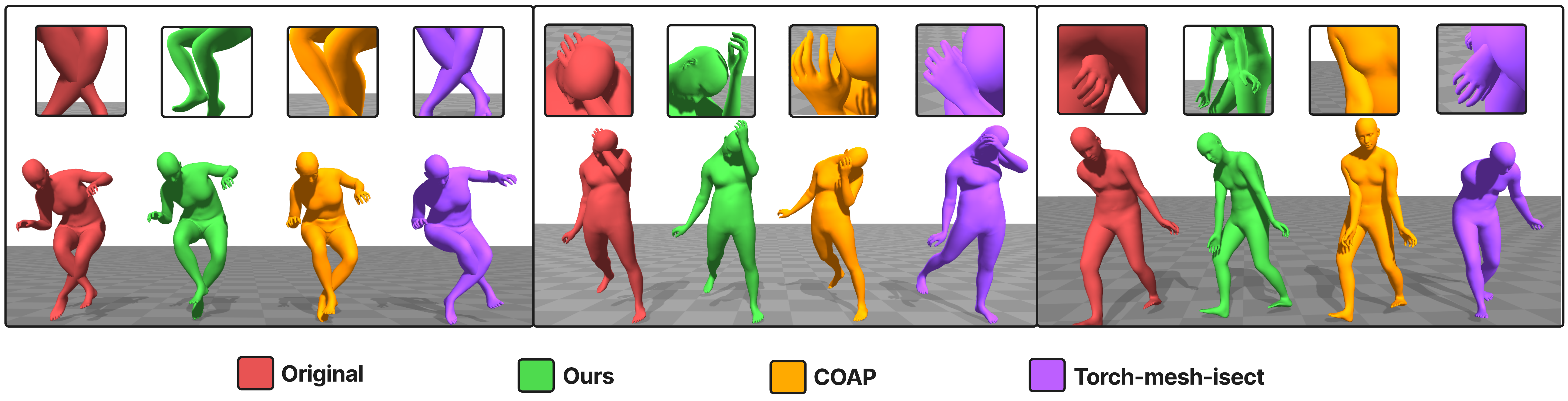}
\caption{Qualitative comparison with baseline methods on three cases (left to right). Within each case, results are shown from left to right: Original input, Ours, COAP~\cite{mihajlovic2022coap}, and Torch-mesh-isect~\cite{tzionas2016capturing}.
Our method consistently removes self-collisions.
Insets highlight representative local self-collision regions.
Torch-mesh-isect fails to resolve the collisions in all three cases,
while COAP almost resolves the first case but still leaves minor residual intersections.}
\vspace{-4mm}
\label{fig:pose}
\end{figure}

\paragraph{Qualitative Results.}
\cref{fig:pose} presents qualitative comparisons with baseline collision-handling methods. Across the three examples, our method consistently removes the highlighted self-intersections, producing visually collision-free configurations in the inset regions. In contrast, Torch-mesh-isect fails to resolve the collisions in all three cases, which is consistent with its low success rate reported in \cref{tab:main_result}. COAP is able to reduce penetrations in some cases and can nearly eliminate the collision in relatively simple scenarios (e.g., the first case), although small residual intersections often remain. This behavior is also reflected in \cref{tab:main_result}, where COAP achieves relatively high PDR but still falls short of fully resolving collisions in some cases. In more complex cases involving multiple body contacts (e.g., the second and third cases), COAP is unable to remove the intersections. Overall, our method achieves more reliable collision resolution while preserving the overall pose structure of the input.

\subsection{Ablation Study}
\label{sec:ablation}
\paragraph{Choices of $d_{\mathrm{SMPL}}(\boldsymbol{\theta},\boldsymbol{\theta}')$.}
We ablate the choice of the pose distance metric. We compare the standard $L_2$ distance:
\begin{equation}
    d_{\mathrm{std}}(\boldsymbol{\theta}, \boldsymbol{\theta}') = 
    \frac{1}{J}\sum_{j=1}^{J} || \boldsymbol{\theta}_j - \boldsymbol{\theta}'_j ||_2,
\end{equation}
against a weighted $L_2$ distance:
\begin{equation}
d_{\mathrm{WD}}(\boldsymbol{\theta}, \boldsymbol{\theta}') = \frac{1}{\sum_{j=1}^{J} w_j}\sum_{j=1}^{J} w_j || \boldsymbol{\theta}_j - \boldsymbol{\theta}'_j ||_2,
\end{equation}
where $\boldsymbol{\theta}_j \in \mathbb{R}^6$ denotes the 6D parameters of the $j$-th joint, and $w_j$ represents the weight assigned based on the size of the subtree rooted at joint $j$ within the kinematic hierarchy. We refer to the former as ``Ours w/o WD'' (Weighted Distance). As shown in \cref{tab:main_result}, while ``Ours w/o WD'' achieves slightly higher collision resolution rates, it suffers from a higher Mean Vertex Distance (MVD). In contrast, our full model with WD significantly reduces MVD from $0.067$ to $0.059$. This reduction is achieved by penalizing rotations of proximal joints more heavily, as their perturbations propagate through the kinematic hierarchy and cause large-scale displacements of downstream subtrees.

\begin{figure*}[t]
\centering
\includegraphics[width=0.7\linewidth]{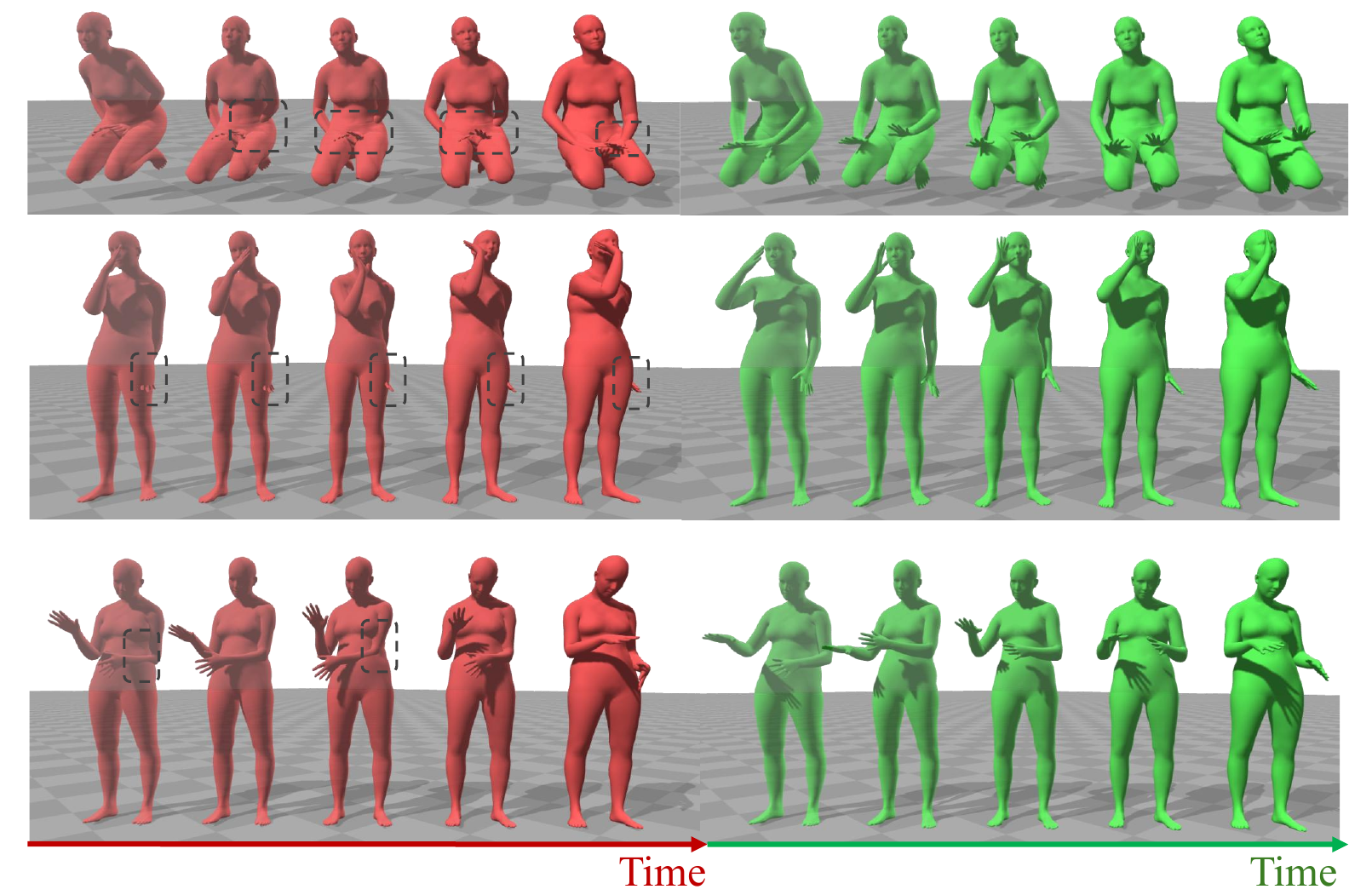}
\caption{Our method resolves the collision in the original samples while preserving the overall motion structure. The \textcolor{myred}{red} ones are original samples, and the \textcolor{mygreen}{green} ones are optimized ones.}
\vspace{-7mm}
\label{fig:motion}
\end{figure*}

\paragraph{$\mathcal{L}_{TD}$ is sufficient.} As shown in \cref{tab:main_result}, using $\mathcal{L}_{TD}$ alone yields the best overall performance across all metrics. We conjecture that incorporating $\mathcal{L}_{grad}$ introduces second-order derivatives during optimization, which increases training instability and hinders convergence. In contrast, $\mathcal{L}_{TD}$ effectively enforces the local geometric consistency. The variant that omits both loss terms fails to approximate a solution to the Eikonal equation, leading to poor performance. 

\begin{figure}[t]
  \centering
  \begin{subfigure}{0.48\textwidth}
    \centering
    \includegraphics[width=\linewidth]{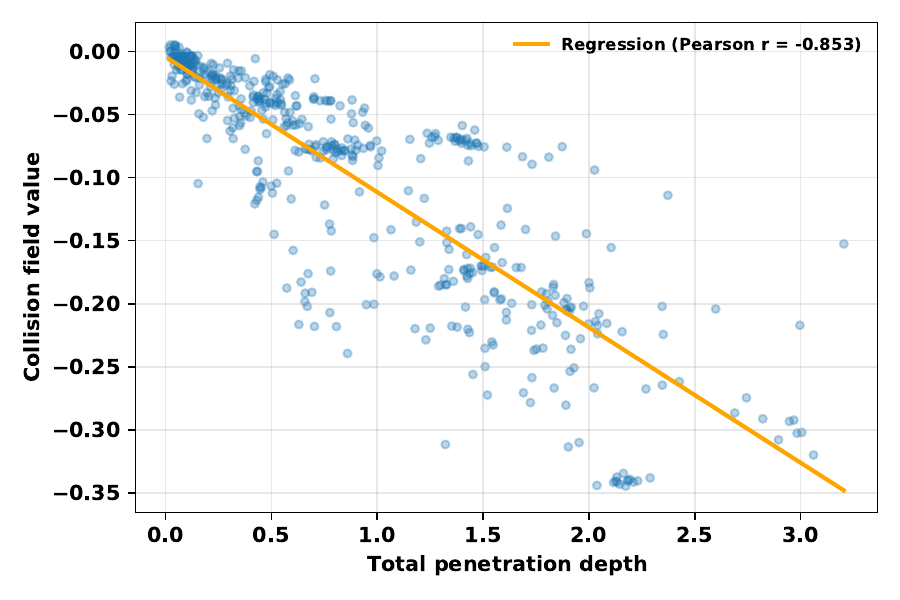}
    \caption{Correlation between $g$ values and PD.}
    \label{fig:cor}
  \end{subfigure}
  \hfill 
  \begin{subfigure}{0.48\textwidth}
    \centering
    \includegraphics[width=\linewidth]{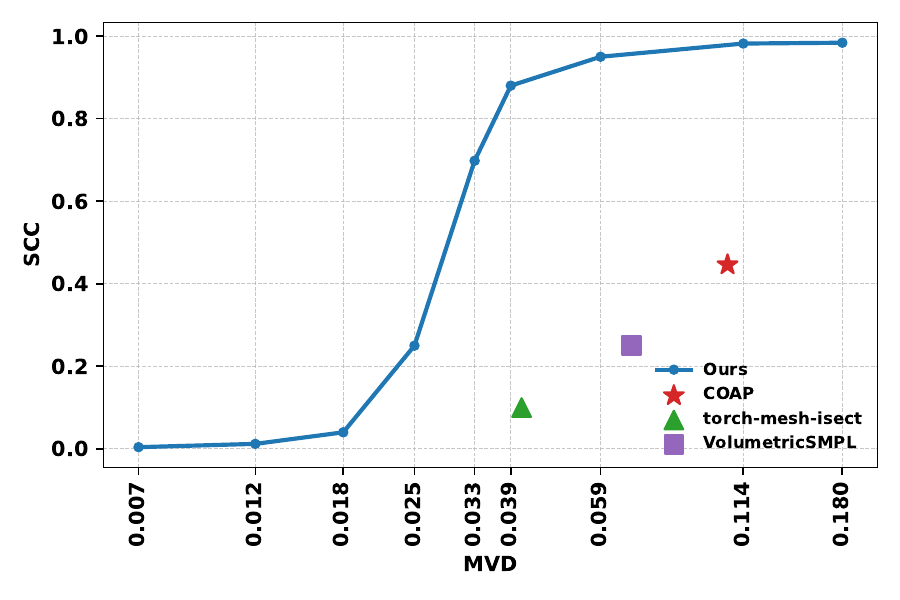}
    \caption{Trade-off between SCC and MVD.}
    \label{fig:Tradeoff}
  \end{subfigure}

  \caption{\textbf{Analysis of the neural collision constraint.} (a) Our neural field $g(\boldsymbol{\theta})$ shows a strong correlation with physical penetration depth (PD). We sample $500$ instances for visualization. (b) The threshold $C_l$ provides a controllable trade-off between collision resolution success (SCC) and geometric fidelity (MVD). The points correspond to $C_l \in \{-0.2, -0.1, -0.05, 0, 0.05, 0.1, 0.2, 0.4, 0.6\}$, ordered left to right. Performance from baseline methods is also included for reference.}
  \label{fig:analysis_combined}
  \vspace{-6mm}
\end{figure}
\paragraph{Tradeoff between SCC and MVD.} As shown in \cref{fig:Tradeoff}, increasing the constraint margin $C_l$ in
\prettyref{eq:smpl_collisionresolution_smooth_main} leads to higher SCC while also
increasing MVD. This reflects the inherent trade-off
between collision resolution and geometric fidelity, as also observed
in~\cite{tan2022n}. Importantly, our method provides a controllable mechanism
to navigate this trade-off simply by tuning $C_l$. As illustrated in
\cref{fig:Tradeoff}, our approach consistently achieves higher SCC at the
same or lower level of MVD compared with baseline methods, indicating a
more favorable trade-off between collision removal and motion preservation.

\paragraph{Correlation between $g$ and penetration depth.} Our method can learn the degree of collision, even though the dataset provides only binary collision labels. While this latent function is not directly visualizable, \cref{fig:cor} shows that its values are correlated with the total penetration depth on the training set.

\subsection{\label{sec:motion_seq}Application: Human Motion Sequence}
We obtain the trained $g$ from \cref{sec:human_pose} and show that $g$ can be robustly generalized to human motion sequences. We utilize a pre-trained motion model~\cite{hymotion2025} as the generative prior $f$ and select motion sequences with self-penetration from a motion dataset~\cite{athanasiou2024motionfix}. Examples are presented in~\cref{fig:motion}. Our method effectively resolves collisions while preserving the overall motion and avoiding noticeable artifacts. More quantitative results are provided in the supplementary material.
\section{Conclusion}
We presented \ours, a neural collision constraint defined directly in SMPL pose space for post-hoc self-collision resolution.
By establishing a connection between collision handling and the Eikonal equation, we provide theoretical grounding for neural constraint learning. In particular, we showed that Eikonal-regularized constraint functions satisfy the LICQ, ensuring the feasibility and numerical stability of constrained optimization in the pose space. The same learned constraint further serves as a generator-agnostic post-hoc corrector for human motion sequences, requiring no retraining of the underlying motion model. Experiments validate that \ours significantly improves collision resolution success rates compared to prior state-of-the-art approaches.

\clearpage
\setcounter{page}{1}

\begin{center}
{\LARGE \bfseries PoseShield: Neural Collision Fields for \\[2pt]
Human Self-Collision Resolution \\[6pt]
\Large --- Supplementary Material ---}
\end{center}
\vspace{1.5em}

% Restart section numbering with letters (A, B, C, ...) for the appendix
\setcounter{section}{0}
\renewcommand{\thesection}{\Alph{section}}

\section{Humans with Collisions Dataset}
\label{sec:hwc}
\begin{figure*}[h] % or [htbp]
\centering
\includegraphics[width=\textwidth]{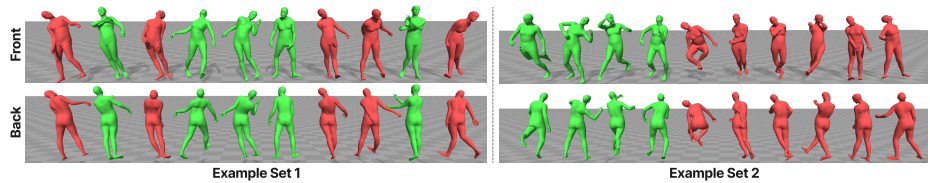}
\caption{\textbf{Twenty randomly selected samples from the HwC Dataset.} 
    The samples are split into two example sets with ten poses each. 
    \textcolor{red}{Red} indicates poses with collisions, while \textcolor{green}{green} denotes collision-free poses.}
\label{fig:human_samples}
\end{figure*}

\paragraph{Human Pose Collision Labeling.}
In human pose representations, self-intersections frequently appear in regions such as the underarm or behind the knees. Although these are technically self-collisions, they primarily arise from the well-known artifacts of linear blend skinning (LBS) and do not affect the perceived physical plausibility of the motion. In contrast, collisions that truly degrade motion realism typically involve interactions across distinct body parts (e.g., hand–body, hand–leg, or left–right leg contacts).
Therefore, when detecting collisions, we exclude triangle–triangle pairs whose topological geodesic distance is less than $50$, as these are considered local artifacts rather than meaningful penetrations. This consideration has been detailed in SMPL-X~\cite{SMPL-X}. For the remaining pairs, a human pose is labeled as “colliding” if any self-collision is detected; otherwise, it is labeled as “non-colliding.”
\begin{figure*}[h] % or [htbp]
\centering
\includegraphics[width=\textwidth]{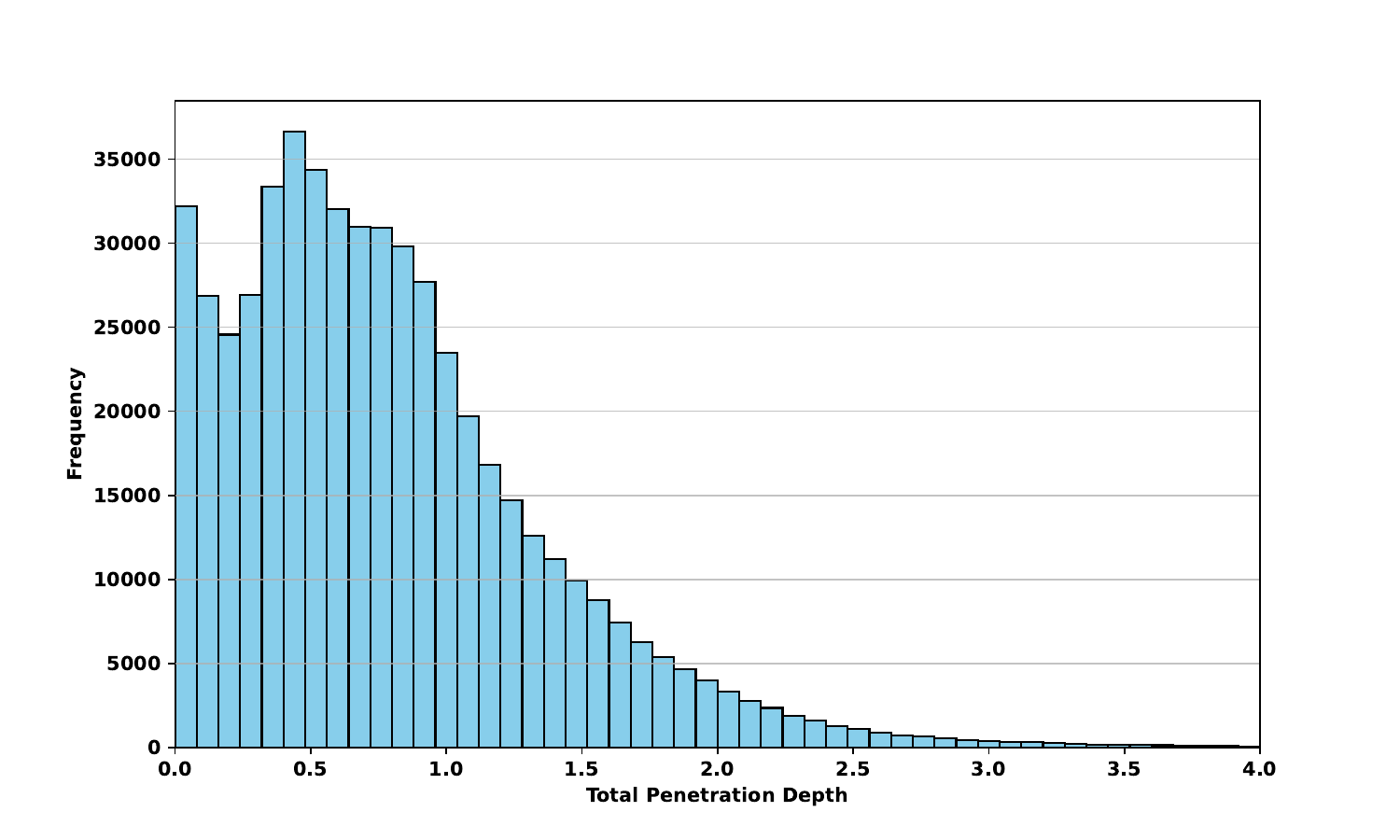}
\caption{Penetration depth distribution of the HwC dataset. Only self-colliding samples are included in the statistics. Non-colliding poses are excluded.}
\label{fig:penetration_depth_dist}
\end{figure*}

\paragraph{Dataset.}
To ensure that the synthesized colliding poses remain close to the natural distribution of valid human poses, we use MotionFix~\cite{athanasiou2024motionfix} as a seed set of meshes and augment it by adding Gaussian noise and applying Gram–Schmidt orthonormalization to generate self-intersecting samples.
Specifically, we adopt the SMPL~\cite{loper2023smpl} parametric space without global translation or global rotation, and convert all remaining joint rotations to the 6D representation.
In total, the resulting latent space has $21 \times 6 = 126$ dimensions. Using this strategy, we obtain a dataset of $931k$ poses.
Among them, $531k$ poses (57\%) exhibit self-collisions, while $399k$ poses (43\%) are collision-free. The dataset is split into training and test sets with a $9:1$ ratio. We further analyze the penetration depth statistics of the generated poses.
The distribution of penetration depth is shown in \cref{fig:penetration_depth_dist},
illustrating a wide range of collision severities in the dataset. However, even with only $10\%$ of the data assigned to the test split, evaluating collision resolution performance for all baselines remains computationally expensive.
Following the practice of previous work~\cite{mihajlovic2022coap}, we randomly sample a subset of $500$ self-penetrating meshes from the \textit{HwC test set} for benchmarking.
Examples of the HwC dataset are shown in \cref{fig:human_samples}, which contains a diverse set of plausible human poses.

\section{Theoretical Analysis}
\label{sec:theory_complete_supp}
In this section, we provide a self-contained theoretical analysis. We first present the rigorous constrained-optimization formulation of SMPL self-collision resolution (\cref{sec:problem_full_supp}). Next, we show that a pose-space signed distance function (SDF) to the collision boundary can be defined (\cref{sec:sdf_supp}). Under idealized assumptions on the neural collision field $g$ (\cref{sec:learned_approximation_supp}), we establish convergence guarantees for the resulting optimization (\cref{sec:convergence_supp}). The central practical contribution is to design a training objective for $g$ that encourages these assumptions to hold approximately; we provide theoretical justification for the Eikonal regularization (\cref{sec:eikonal_loss_property_supp}).

\subsection{Problem Formulation}
\label{sec:problem_full_supp}
An SMPL mesh~\cite{loper2023smpl,SMPL-X} is defined by shape parameters $\boldsymbol{\beta} \in \mathbb{R}^{d_\beta}$ and pose parameters $\boldsymbol{\theta} \in \mathbb{R}^{d_\theta}$. 
Given $(\boldsymbol{\beta}, \boldsymbol{\theta})$, the SMPL function produces a mesh:
\[
X = \mathcal{M}(\boldsymbol{\beta}, \boldsymbol{\theta}),
\]
where the mesh connectivity $\mathcal{T}$ is fixed and defined by the SMPL function itself. Since global translation and rotation don't affect self-collisions, we ignore them and let $ \boldsymbol{\theta}\in \mathbb{R}^{J\times 6}$ represent the $6$D rotations~\cite{zhou2019continuity} of the $J$ joints. We assume the shape parameter $\boldsymbol{\beta}$ is fixed and only optimize the pose parameter $\boldsymbol{\theta}$. 
\begin{definition}[$6$D pose domain]
For one joint, write the $6$D representation as
$\mathbf{r}=(\mathbf{a},\mathbf{b})\in\mathbb{R}^3\times\mathbb{R}^3$.  Its non-degenerate domain is
\begin{equation}
\label{eq:sixd_single_joint_domain_supp}
    \mathcal{D}
    :=
    \left\{(\mathbf{a},\mathbf{b})
    \;\middle|\;
    \|\mathbf{a}\|_2>0,\;
    \left\|\mathbf{b}-(\mathbf{u}^{\top}\mathbf{b})\mathbf{u}\right\|_2>0,
    \;\mathbf{u}=\frac{\mathbf{a}}{\|\mathbf{a}\|_2}
    \right\}.
\end{equation}
On $\mathcal{D}$, the Gram--Schmidt map $\pi_{6D}:\mathcal{D}\to SO(3)$ is
\begin{equation}
\label{eq:sixd_gs_map_supp}
    \pi_{6D}(\mathbf{r})
    =
    \big[\mathbf{u}\;\mathbf{v}\;\mathbf{u}\times\mathbf{v}\big],
    \qquad
    \mathbf{u}=\frac{\mathbf{a}}{\|\mathbf{a}\|_2},
    \quad
    \mathbf{v}=\frac{\mathbf{b}-(\mathbf{u}^{\top}\mathbf{b})\mathbf{u}}
    {\|\mathbf{b}-(\mathbf{u}^{\top}\mathbf{b})\mathbf{u}\|_2}.
\end{equation}
\end{definition}

In our setting, we restrict the $6$D coordinates to a bounded region
\begin{equation}
\label{eq:bounded_pose_box_supp}
    \Omega_B := [-B,B]^{J\times 6}\subset\mathbb{R}^{J\times 6},
    \qquad B>1.
\end{equation}
Thus the optimized pose variable is the concatenated $6$D vector
\begin{equation}
\label{eq:sixd_pose_domain_supp}
    \boldsymbol{\theta}
    =
    (\mathbf{r}_1,\ldots,\mathbf{r}_J)
    \in
    \Theta := \mathcal{D}^J\cap\Omega_B.
\end{equation}
The associated output of the Gram--Schmidt map is not $6$D; it is a tuple of rotation matrices
$(\pi_{6D}(\mathbf{r}_1),\ldots,\pi_{6D}(\mathbf{r}_J))\in SO(3)^J$, which is what SMPL uses to produce the mesh.

\begin{definition}[Extended exact collision indicator]
For a fixed shape $\boldsymbol{\beta}$, we define the exact collision indicator on the full
bounded pose box as a map
$\iota_{\boldsymbol{\beta}}:\Omega_B\to\{-1,+1\}$.  For non-degenerate poses
$\boldsymbol{\theta}\in\Theta$, it is obtained by decoding the mesh
$X=\mathcal{M}(\boldsymbol{\beta},\boldsymbol{\theta})$ and applying an exact mesh
self-intersection test, e.g., a classical collision detector such as FCL~\cite{pan2012fcl}:
\begin{equation}
\label{eq:binary_collision_indicator_supp}
    \iota_{\boldsymbol{\beta}}(\boldsymbol{\theta})
    :=
    \begin{cases}
    \iota\big(\mathcal{M}(\boldsymbol{\beta},\boldsymbol{\theta})\big),
    & \boldsymbol{\theta}\in\Theta,\\
    -1,
    & \boldsymbol{\theta}\in\Omega_B\setminus\Theta,
    \end{cases}
    \qquad \boldsymbol{\theta}\in\Omega_B.
\end{equation}
\end{definition}

Here $\iota_{\boldsymbol{\beta}}(\boldsymbol{\theta})=-1$ denotes a colliding or degenerate input,
and $\iota_{\boldsymbol{\beta}}(\boldsymbol{\theta})=+1$ denotes a non-degenerate collision-free
pose.  The exact collision-free pose set for the fixed shape $\boldsymbol{\beta}$ is
\begin{equation}
\label{eq:exact_feasible_set_supp}
    \mathcal{F}_{\boldsymbol{\beta}}
    :=
    \{\boldsymbol{\theta}\in\Omega_B
    \mid
    \iota_{\boldsymbol{\beta}}(\boldsymbol{\theta})=+1\}.
\end{equation}
Thus all degenerate inputs in $\Omega_B\setminus\Theta$ are treated as infeasible and are
outside $\mathcal{F}_{\boldsymbol{\beta}}$ by definition.

Given a self-colliding SMPL configuration $(\boldsymbol{\beta},\boldsymbol{\theta}_0)$ with
$\iota_{\boldsymbol{\beta}}(\boldsymbol{\theta}_0)=-1$, our goal is to find a corrected pose
$\boldsymbol{\theta}$ whose decoded mesh is collision-free while remaining close to the original
configuration.  Let $d_{\mathrm{SMPL}}(\cdot,\cdot)$ denote a pose discrepancy measure for a fixed
shape.  We formulate SMPL self-collision resolution as
\begin{equation}
\label{eq:smpl_collisionresolution_supp}
\boldsymbol{\theta}^\star
=
\arg\min_{\boldsymbol{\theta}\in\Theta}
    d_{\mathrm{SMPL}}(\boldsymbol{\theta},\boldsymbol{\theta}_0)
\quad
\text{subject to}
\quad
    \boldsymbol{\theta}\in\mathcal{F}_{\boldsymbol{\beta}}.
\end{equation}
In practice, the starting pose $\boldsymbol{\theta}_0$ is a normalized $6$D rotation, hence $\boldsymbol{\theta}_0\in[-1,1]^{J\times 6}$.

\subsection{Indicator-Induced Signed Distance Function}
\label{sec:sdf_supp}
We show that the collision indicator induces an SDF in $\Omega_B$. 
\begin{definition}[Indicator-induced pose-space SDF]
The extended exact collision indicator induces the infeasible set
\begin{equation}
\label{eq:indicator_induced_sets_supp}
    \mathcal{C}_{\boldsymbol{\beta}}
    :=
    \iota_{\boldsymbol{\beta}}^{-1}(-1)
    =
    \Omega_B\setminus\mathcal{F}_{\boldsymbol{\beta}}.
\end{equation}
For any nonempty set $A\subset\Omega_B$, define
\begin{equation}
\label{eq:distance_to_set_supp}
    \operatorname{dist}(\boldsymbol{\theta},A)
    :=
    \inf_{\mathbf{y}\in A}\|\boldsymbol{\theta}-\mathbf{y}\|_2,
    \qquad \boldsymbol{\theta}\in\Omega_B .
\end{equation}
When both $\mathcal{F}_{\boldsymbol{\beta}}$ and $\mathcal{C}_{\boldsymbol{\beta}}$ are nonempty, we define the indicator-induced relative signed distance function by
\begin{equation}
\label{eq:true_pose_space_sdf_supp}
    \phi_{\boldsymbol{\beta}}(\boldsymbol{\theta})
    :=
    \operatorname{dist}\big(\boldsymbol{\theta},\mathcal{C}_{\boldsymbol{\beta}}\big)
    -
    \operatorname{dist}\big(\boldsymbol{\theta},\mathcal{F}_{\boldsymbol{\beta}}\big),
    \qquad \boldsymbol{\theta}\in\Omega_B .
\end{equation}
\end{definition}
\begin{theorem}[Properties of the indicator-induced SDF]\label{thm:indicator_sdf_properties_supp} The function $\phi_{\boldsymbol{\beta}}$ in~\eqref{eq:true_pose_space_sdf_supp} is well-defined on
$\Omega_B$ and satisfies
\begin{equation}
\label{eq:indicator_induced_sdf_sign_supp}
    \begin{cases}
    \phi_{\boldsymbol{\beta}}(\boldsymbol{\theta})\geq 0,
    & \boldsymbol{\theta}\in\mathcal{F}_{\boldsymbol{\beta}},\\
    \phi_{\boldsymbol{\beta}}(\boldsymbol{\theta})\leq 0,
    & \boldsymbol{\theta}\in\mathcal{C}_{\boldsymbol{\beta}}.
    \end{cases}
\end{equation}
Moreover,
$\phi_{\boldsymbol{\beta}}$ is Lipschitz continuous and differentiable almost everywhere.  At
differentiability points away from the zero level set where the relevant closest point is unique,
it satisfies the Eikonal property
\begin{equation}
\label{eq:indicator_induced_eikonal_supp}
    \|\nabla\phi_{\boldsymbol{\beta}}(\boldsymbol{\theta})\|_2=1.
\end{equation}
\end{theorem}

\begin{proof}
By the nonemptiness condition in Definition~3, the two distance-to-set terms are finite on
$\Omega_B$.  The sign property follows
from $\operatorname{dist}(\boldsymbol{\theta},\mathcal{F}_{\boldsymbol{\beta}})=0$ for
$\boldsymbol{\theta}\in\mathcal{F}_{\boldsymbol{\beta}}$ and
$\operatorname{dist}(\boldsymbol{\theta},\mathcal{C}_{\boldsymbol{\beta}})=0$ for
$\boldsymbol{\theta}\in\mathcal{C}_{\boldsymbol{\beta}}$.  Each distance-to-set function is
Lipschitz continuous, so $\phi_{\boldsymbol{\beta}}$ is Lipschitz continuous; by Rademacher's
theorem, it is differentiable almost everywhere.  Finally, on either side of the zero level set,
$\phi_{\boldsymbol{\beta}}$ locally reduces to either
$\operatorname{dist}(\cdot,\mathcal{C}_{\boldsymbol{\beta}})$ or
$-\operatorname{dist}(\cdot,\mathcal{F}_{\boldsymbol{\beta}})$.  The standard Euclidean distance
function satisfies
$\nabla\operatorname{dist}(\boldsymbol{\theta},A)=(\boldsymbol{\theta}-\mathbf{p})/
\|\boldsymbol{\theta}-\mathbf{p}\|_2$ at differentiability points with unique closest point
$\mathbf{p}\in \overline{A}$, and hence has unit gradient norm.  This proves the Eikonal property
in~\eqref{eq:indicator_induced_eikonal_supp}.
\end{proof}
Suppose the SMPL distance is simply the Euclidean distance in the
$6$D pose coordinates,
\begin{equation}
\label{eq:special_case_euclidean_distance_supp}
    d_{\mathrm{SMPL}}(\boldsymbol{\theta},\boldsymbol{\theta}_0)
    =
    \|\boldsymbol{\theta}-\boldsymbol{\theta}_0\|_2^2.
\end{equation}

For an infeasible pose $\boldsymbol{\theta}_0\in\mathcal{C}_{\boldsymbol{\beta}}$, the optimal value of \cref{eq:smpl_collisionresolution_supp} equals the distance to the closure $\overline{\mathcal{F}_{\boldsymbol{\beta}}}$~\cite{marz2012calculus}
\begin{equation}
\label{eq:special_case_closest_point_supp}
    \inf_{\boldsymbol{\theta}\in\mathcal{F}_{\boldsymbol{\beta}}}\|\boldsymbol{\theta}-\boldsymbol{\theta}_0\|_2
    =
    \min_{\boldsymbol{\theta}\in\overline{\mathcal{F}_{\boldsymbol{\beta}}}}
    \|\boldsymbol{\theta}-\boldsymbol{\theta}_0\|_2,
\end{equation}
which is attained in $\mathcal{F}_{\boldsymbol{\beta}}$ (and then solves \cref{eq:smpl_collisionresolution_supp}) when $\mathcal{F}_{\boldsymbol{\beta}}$ is closed.

\subsection{Assumptions on the Neural Collision Field}
\label{sec:learned_approximation_supp}
The exact SDF $\phi_{\boldsymbol{\beta}}$ is intractable in the high-dimensional space $\Omega_B$. To obtain a differentiable surrogate, we learn a
neural collision field on the full bounded $6$D box
\begin{equation}
\label{eq:neural_collision_field_supp}
    g_{\boldsymbol{\beta}}:\Omega_B\rightarrow\mathbb{R}.
\end{equation}
 Since the shape is clear from context, we denote it by $g(\boldsymbol{\theta})$. For the following formulation, we state three idealized assumptions on
$g$.

\begin{assumption}[Smoothness]\label{asmp:smoothness_supp}
The learned field is twice continuously differentiable on the bounded domain, with Lipschitz-continuous gradient and Hessian:
\begin{equation}
\label{eq:assumption_smoothness_supp}
    g\in C^2(\Omega_B).
\end{equation}
\end{assumption}
This is satisfied by a multi-layer perceptron (MLP) with standard Softplus activations, which is
in fact smooth.

\begin{assumption}[Feasibility consistency]\label{asmp:feasibility_supp}
The non-negative superlevel set of the learned field exactly recovers the collision-free set:
\begin{equation}
\label{eq:assumption_conservative_feasibility_supp}
    \mathcal{F}
    =
    \{\boldsymbol{\theta}\in\Omega_B
    \mid
    g(\boldsymbol{\theta})\geq0\}.
\end{equation}
Equivalently, for non-degenerate inputs,
\begin{equation}
\label{eq:assumption_indicator_implications_supp}
    \iota_{\boldsymbol{\beta}}(\boldsymbol{\theta})=+1
    \Rightarrow
    g(\boldsymbol{\theta})\geq0,
    \qquad
    \iota_{\boldsymbol{\beta}}(\boldsymbol{\theta})=-1
    \Rightarrow
    g(\boldsymbol{\theta})< 0,
    \quad
    \boldsymbol{\theta}\in\Theta,
\end{equation}
and all degenerate inputs are also infeasible:
\begin{equation}
\label{eq:assumption_degenerate_nonpositive_supp}
    \boldsymbol{\theta}\in\Omega_B\setminus\Theta
    \Rightarrow
    g(\boldsymbol{\theta})< 0.
\end{equation}
Thus $g\ge 0$ denotes collision-free poses, while $g<0$
denotes colliding or degenerate inputs.
\end{assumption}

\begin{assumption}[Approximate Eikonal property]\label{asmp:approx_eikonal_supp}
    There exists a constant $\delta\in[0,1)$ such that the learned field satisfies
\begin{equation}
\label{eq:assumption_approx_eikonal_supp}
    1-\delta
    \leq
    \|\nabla_{\boldsymbol{\theta}} g(\boldsymbol{\theta})\|_2
    \leq
    1+\delta,
    \qquad \forall\boldsymbol{\theta}\in\Omega_B.
\end{equation}
\end{assumption}
This assumption makes $g$ an approximate signed distance function (SDF) in pose
space: its gradient is non-vanishing near the collision boundary, and
$|g_{\boldsymbol{\beta}}(\boldsymbol{\theta})|$ can be interpreted as an approximate distance to the
learned boundary $g(\boldsymbol{\theta})=0$.

With this surrogate constraint, the collision-resolution problem becomes
\begin{equation}
\label{eq:smpl_collisionresolution_smooth_supp}
\boldsymbol{\theta}^\star
=
\arg\min_{\boldsymbol{\theta}\in\Theta}
    d_{\mathrm{SMPL}}(\boldsymbol{\theta},\boldsymbol{\theta}_0)
\quad
\text{subject to}
\quad
    g(\boldsymbol{\theta}) \geq 0 .
\end{equation}

The training objective for $g$ is designed to encourage these assumptions to hold approximately.

\subsection{Convergence Analysis}
\label{sec:convergence_supp}
Assuming a $g$ that satisfies the assumptions above, we establish the convergence properties of the SLSQP algorithm in \cref{eq:smpl_collisionresolution_smooth_supp}. For simplicity, the subsequent analysis assumes that all iterates $\boldsymbol{\theta}_k$ remain bounded within $\Omega_B$. Throughout, we take the squared-Euclidean objective $d_{\mathrm{SMPL}}(\boldsymbol{\theta},\boldsymbol{\theta}_0)=\|\boldsymbol{\theta}-\boldsymbol{\theta}_0\|^2$ and assume the minimizer is interior to $\Omega_B$ and non-degenerate, so that $g\ge 0$ is the only active constraint. We recall that $\boldsymbol{\theta}_0\in[-1,1]^{J\times 6}$. If necessary, $B$ could be expanded to a sufficiently large value, such as $100$, yielding analogous results.

\begin{theorem}[Global Convergence and Complexity]\label{thm:global_supp}
Consider problem~\eqref{eq:smpl_collisionresolution_smooth_supp} under Assumptions~\ref{asmp:smoothness_supp},~\ref{asmp:feasibility_supp} and~\ref{asmp:approx_eikonal_supp}. Then:
\begin{enumerate}
\item[\textup{(i)}] \textbf{Global LICQ:} The constraint qualification holds globally on $\Omega_B$.
\item[\textup{(ii)}] \textbf{Global Convergence:} From any starting pose $\boldsymbol{\theta}_{0} \in \Omega_B$, a standard line-search SQP method with an $\ell_1$ merit function (and sufficiently large penalty parameter) produces iterates whose every accumulation point is a first-order KKT point $(\boldsymbol{\theta}^\star, \lambda^\star)$.
\item[\textup{(iii)}] \textbf{Iteration Complexity:} An $\varepsilon$-approximate KKT point---satisfying
$$\bigl\| 2(\boldsymbol{\theta}_k - \boldsymbol{\theta}_0) - \lambda_k \nabla_{\boldsymbol{\theta}} g(\boldsymbol{\theta}_k) \bigr\| \leq \varepsilon, \quad |\min(0,\, g(\boldsymbol{\theta}_k))| \leq \varepsilon, \quad \lambda_k \geq 0, \quad |\lambda_k\, g(\boldsymbol{\theta}_k)| \leq \varepsilon,$$
is no harder to obtain than in unconstrained smooth optimization, whose worst-case first-order complexity is $\mathcal{O}(\varepsilon^{-2})$.
\end{enumerate}
\end{theorem}
\begin{proof}

We establish the three claims in sequence.

\textbf{Part (i): Global LICQ.}\;
For a single inequality constraint \(g(\boldsymbol{\theta}) \geq 0\), LICQ requires that the gradient of the active constraint be nonzero. The approximate Eikonal assumption gives
\[
\|\nabla_{\boldsymbol{\theta}} g(\boldsymbol{\theta})\| \geq 1 - \delta > 0
\qquad
(\text{since } \delta < 1)
\]
for all \(\boldsymbol{\theta} \in \Omega_B\), so LICQ holds strictly and globally.

\textbf{Part (ii): Global Convergence.}\;
We use SQP with the \(\ell_1\) exact penalty merit function
\[
\phi(\boldsymbol{\theta};\,\mu)
=
d_{\mathrm{SMPL}}(\boldsymbol{\theta},\boldsymbol{\theta}_0)
+
\mu \max(0,\,-g(\boldsymbol{\theta})),
\]
where \(\mu > \tfrac{2}{1-\delta}\operatorname{diam}(\Omega_B)\) is a penalty parameter. At each iteration, a search direction \(d_k\) is obtained by solving a QP subproblem that linearizes the constraint, and a line search on \(\phi\) ensures progress.

The primary failure mode in nonconvex settings is stagnation at an \emph{infeasible stationary point}: a point \(\boldsymbol{\theta}\) where \(g(\boldsymbol{\theta}) < 0\) but \(\nabla_{\boldsymbol{\theta}} g(\boldsymbol{\theta}) = 0\). At such a point, the linearized feasibility condition
\[
g(\boldsymbol{\theta}_k)
+
\nabla_{\boldsymbol{\theta}} g(\boldsymbol{\theta}_k)^\top d
\geq 0
\]
reduces to the false statement \(g(\boldsymbol{\theta}_k) \geq 0\), so no direction \(d\) can improve feasibility in the linear model and the QP subproblem degenerates.

The approximate Eikonal assumption eliminates this failure mode. At any infeasible point \(\boldsymbol{\theta}\) with \(g(\boldsymbol{\theta}) < 0\), the direction
\[
d = t\,\nabla_{\boldsymbol{\theta}} g(\boldsymbol{\theta}),
\qquad t>0,
\]
satisfies
\[
g(\boldsymbol{\theta})
+
\nabla_{\boldsymbol{\theta}} g(\boldsymbol{\theta})^\top d
=
g(\boldsymbol{\theta})
+
t\,\|\nabla_{\boldsymbol{\theta}} g(\boldsymbol{\theta})\|^2
\geq
g(\boldsymbol{\theta}) + t(1-\delta)^2,
\]
which is nonnegative for
\[
t \geq \frac{|g(\boldsymbol{\theta})|}{(1-\delta)^2}.
\]
Thus, the linearized constraint always admits a feasible direction and the QP subproblem is always strictly feasible.

Since \(\|\nabla_{\boldsymbol{\theta}} d_{\mathrm{SMPL}}\|=2\|\boldsymbol{\theta}-\boldsymbol{\theta}_0\|\le 2\operatorname{diam}(\Omega_B)\), the choice of \(\mu\) gives \(\|\nabla_{\boldsymbol{\theta}}\phi(\boldsymbol{\theta};\mu)\|\ge \mu(1-\delta)-2\operatorname{diam}(\Omega_B)>0\) whenever \(g(\boldsymbol{\theta})<0\); hence \(\phi(\cdot;\mu)\) has no infeasible stationary point and is exact.

Moreover,
\[
d_{\mathrm{SMPL}}(\boldsymbol{\theta},\boldsymbol{\theta}_0)
=
\|\boldsymbol{\theta}-\boldsymbol{\theta}_0\|^2
\]
is coercive, so the sublevel set
\[
\left\{
\boldsymbol{\theta} :
\phi(\boldsymbol{\theta};\,\mu)
\leq
\phi(\boldsymbol{\theta}_{0};\,\mu)
\right\}
\]
is compact and the iterates \(\{\boldsymbol{\theta}_k\}\) remain bounded.

With (i) bounded iterates, (ii) Lipschitz continuous gradients and Hessians, and (iii) uniformly full-rank constraint Jacobian
\[
\|\nabla_{\boldsymbol{\theta}} g(\boldsymbol{\theta})\|
\geq
1-\delta > 0
\qquad
\text{for all } \boldsymbol{\theta}\in\Omega_B,
\]
The hypotheses of the standard SQP global convergence theorem are satisfied. The line search ensures sufficient decrease of \(\phi\) at each iteration, and every limit point of \(\{\boldsymbol{\theta}_k\}\) is a first-order KKT point.

\textbf{Part (iii): \(\mathcal{O}(\epsilon^{-2})\) Iteration Complexity.}\;
In unconstrained nonconvex optimization with \(L\)-Lipschitz gradient, gradient descent requires at most \(\mathcal{O}(\epsilon^{-2})\) iterations to find an \(\epsilon\)-stationary point.

In constrained optimization, the complexity additionally depends on the conditioning of the constraint Jacobian
\[
J(\boldsymbol{\theta})
=
\nabla_{\boldsymbol{\theta}} g(\boldsymbol{\theta})^\top
\in
\mathbb{R}^{1\times N}.
\]
Its minimum singular value \(\sigma_{\min}(J(\boldsymbol{\theta}))\) governs how effectively the solver projects steps onto the feasible region. If \(\sigma_{\min}\to 0\), the penalty parameter \(\mu\) must grow as \(\mathcal{O}(1/\sigma_{\min})\) to enforce feasibility, step sizes shrink, and complexity degrades.

Under the approximate Eikonal assumption,
\[
J(\boldsymbol{\theta})J(\boldsymbol{\theta})^\top
=
\|\nabla_{\boldsymbol{\theta}} g(\boldsymbol{\theta})\|^2
\geq
(1-\delta)^2,
\]
so
\[
\sigma_{\min}(J(\boldsymbol{\theta}))
\geq
1-\delta > 0
\qquad
\text{for all } \boldsymbol{\theta}\in\Omega_B.
\]
The constraint Jacobian therefore has a singular value bounded uniformly away from zero over the entire domain. Consequently, the penalty parameter \(\mu\) remains \(\mathcal{O}(1/(1-\delta))\), and the constrained problem inherits the worst-case evaluation complexity of unconstrained smooth optimization. 
\end{proof}

\begin{theorem}[Local Convergence]\label{thm:local_supp}
Consider problem~\eqref{eq:smpl_collisionresolution_smooth_supp} under Assumptions~\ref{asmp:smoothness_supp},~\ref{asmp:feasibility_supp} and~\ref{asmp:approx_eikonal_supp}. Suppose the initial pose $\boldsymbol{\theta}_0$ is infeasible ($g(\boldsymbol{\theta}_0) < 0$). Let $\boldsymbol{\theta}^\star$ be a local minimizer. Then:
\begin{enumerate}
\item[\textup{(i)}] \textbf{LICQ and Strict Complementarity:} LICQ holds at $\boldsymbol{\theta}^\star$, and the unique KKT multiplier satisfies $0 < \frac{2}{1+\delta}\|\boldsymbol{\theta}^\star - \boldsymbol{\theta}_0\| \leq \lambda^\star \leq \frac{2}{1-\delta}\|\boldsymbol{\theta}^\star - \boldsymbol{\theta}_0\|$.
\item[\textup{(ii)}] \textbf{SOSC and Fast Convergence:} Define $\kappa \triangleq \lambda^\star\|\nabla_{\boldsymbol{\theta}}^2 g(\boldsymbol{\theta}^\star)\|_2$. If $\kappa < 2$, then the full Lagrangian Hessian is positive definite (implying Second-Order Sufficient Conditions), and SQP with exact Hessian converges locally to $(\boldsymbol{\theta}^\star, \lambda^\star)$ at a quadratic rate. A quasi-Newton (BFGS) variant satisfying the Dennis--Mor\'e condition converges superlinearly.
\end{enumerate}
\end{theorem}
\begin{proof}
\textbf{Part (i): LICQ and Strict Complementarity.}\;
The approximate Eikonal assumption gives
\[
\|\nabla_{\boldsymbol{\theta}} g(\boldsymbol{\theta}^\star)\|
\geq
1-\delta > 0,
\]
so LICQ holds at \(\boldsymbol{\theta}^\star\).

Since \(\boldsymbol{\theta}_0\) is strictly infeasible (\(g(\boldsymbol{\theta}_0)<0\)) and the unconstrained minimizer of \(d_{\mathrm{SMPL}}(\boldsymbol{\theta},\boldsymbol{\theta}_0)\) is \(\boldsymbol{\theta}_0\) itself, any constrained local minimizer must lie on the boundary \(g(\boldsymbol{\theta}^\star)=0\). Moreover, \(\boldsymbol{\theta}^\star \neq \boldsymbol{\theta}_0\) because \(\boldsymbol{\theta}_0\) is infeasible while \(\boldsymbol{\theta}^\star\) is feasible.

Because LICQ holds, the KKT conditions are necessary at \(\boldsymbol{\theta}^\star\). The stationarity condition
\[
\nabla_{\boldsymbol{\theta}} L(\boldsymbol{\theta}^\star,\lambda^\star)=0
\]
requires
\[
2(\boldsymbol{\theta}^\star-\boldsymbol{\theta}_0)
=
\lambda^\star \nabla_{\boldsymbol{\theta}} g(\boldsymbol{\theta}^\star).
\]
Taking norms on both sides gives
\[
2\|\boldsymbol{\theta}^\star-\boldsymbol{\theta}_0\|
=
|\lambda^\star|\,\|\nabla_{\boldsymbol{\theta}} g(\boldsymbol{\theta}^\star)\|.
\]
Using
\[
1-\delta
\leq
\|\nabla_{\boldsymbol{\theta}} g(\boldsymbol{\theta}^\star)\|
\leq
1+\delta,
\]
we obtain
\[
|\lambda^\star|(1-\delta)
\leq
2\|\boldsymbol{\theta}^\star-\boldsymbol{\theta}_0\|
\leq
|\lambda^\star|(1+\delta).
\]
Since \(\boldsymbol{\theta}^\star \neq \boldsymbol{\theta}_0\), we have \(\|\boldsymbol{\theta}^\star-\boldsymbol{\theta}_0\|>0\), which implies \(|\lambda^\star|>0\). Combined with dual feasibility \(\lambda^\star\geq 0\), this yields
\[
0
<
\frac{2}{1+\delta}\|\boldsymbol{\theta}^\star-\boldsymbol{\theta}_0\|
\leq
\lambda^\star
\leq
\frac{2}{1-\delta}\|\boldsymbol{\theta}^\star-\boldsymbol{\theta}_0\|,
\]
establishing strict complementarity.

\textbf{Part (ii): SOSC and Convergence.}\;
The Lagrangian Hessian with respect to \(\boldsymbol{\theta}\) is
\[
\nabla^2_{\boldsymbol{\theta}\boldsymbol{\theta}}
L(\boldsymbol{\theta}^\star,\lambda^\star)
=
\nabla^2_{\boldsymbol{\theta}\boldsymbol{\theta}}
d_{\mathrm{SMPL}}(\boldsymbol{\theta}^\star,\boldsymbol{\theta}_0)
-
\lambda^\star \nabla^2_{\boldsymbol{\theta}} g(\boldsymbol{\theta}^\star)
=
2I - \lambda^\star \nabla^2_{\boldsymbol{\theta}} g(\boldsymbol{\theta}^\star),
\]
where the \(2I\) term comes from
\[
d_{\mathrm{SMPL}}(\boldsymbol{\theta},\boldsymbol{\theta}_0)
=
\|\boldsymbol{\theta}-\boldsymbol{\theta}_0\|^2.
\]

For any nonzero \(v \in \mathbb{R}^{N}\), the spectral norm bound gives
\[
|v^\top \nabla^2_{\boldsymbol{\theta}} g(\boldsymbol{\theta}^\star) v|
\leq
\|\nabla^2_{\boldsymbol{\theta}} g(\boldsymbol{\theta}^\star)\|_2\,\|v\|^2,
\]
and hence
\[
v^\top
\nabla^2_{\boldsymbol{\theta}\boldsymbol{\theta}}
L(\boldsymbol{\theta}^\star,\lambda^\star)
v
\geq
2\|v\|^2
-
\lambda^\star
\|\nabla^2_{\boldsymbol{\theta}} g(\boldsymbol{\theta}^\star)\|_2
\|v\|^2
=
(2-\kappa)\|v\|^2.
\]
When \(\kappa<2\), the coefficient \(2-\kappa>0\), so the full Lagrangian Hessian is positive definite on all of \(\mathbb{R}^{N}\). Since positive definiteness on \(\mathbb{R}^{N}\) implies positive definiteness on any subspace---in particular on the tangent space
\[
\left\{
w : \nabla_{\boldsymbol{\theta}} g(\boldsymbol{\theta}^\star)^\top w = 0
\right\},
\]
the second-order sufficient condition holds.

\textit{Convergence.}\;
Because the constraint is active at \(\boldsymbol{\theta}^\star\) with \(\lambda^\star>0\), the problem locally reduces to the equality-constrained problem
\[
g(\boldsymbol{\theta})=0.
\]
SQP applied to this equality-constrained problem is equivalent to Newton's method on the KKT system \(F(\boldsymbol{\theta},\lambda)=0\), where
\[
F(\boldsymbol{\theta},\lambda)
=
\begin{pmatrix}
2(\boldsymbol{\theta}-\boldsymbol{\theta}_0)
-
\lambda\,\nabla_{\boldsymbol{\theta}} g(\boldsymbol{\theta})
\\[3pt]
g(\boldsymbol{\theta})
\end{pmatrix},
\qquad
J(\boldsymbol{\theta},\lambda)
=
\begin{pmatrix}
2I - \lambda \nabla^2_{\boldsymbol{\theta}} g(\boldsymbol{\theta})
&
-\nabla_{\boldsymbol{\theta}} g(\boldsymbol{\theta})
\\[3pt]
\nabla_{\boldsymbol{\theta}} g(\boldsymbol{\theta})^\top
&
0
\end{pmatrix}.
\]
At \((\boldsymbol{\theta}^\star,\lambda^\star)\), the KKT matrix is nonsingular: LICQ ensures \(\nabla_{\boldsymbol{\theta}} g(\boldsymbol{\theta}^\star)\neq 0\), and SOSC ensures invertibility of the reduced Hessian block. Since \(F\) is continuously differentiable with a Lipschitz Jacobian by Assumption~\ref{asmp:smoothness_supp}, the classical Newton convergence theorem guarantees local quadratic convergence: there exist a neighborhood \(\mathcal{N}\) of \((\boldsymbol{\theta}^\star,\lambda^\star)\) and a constant \(M>0\) such that
\[
\left\|
\begin{pmatrix}
\boldsymbol{\theta}_{k+1}-\boldsymbol{\theta}^\star
\\[2pt]
\lambda_{k+1}-\lambda^\star
\end{pmatrix}
\right\|
\leq
M
\left\|
\begin{pmatrix}
\boldsymbol{\theta}_k-\boldsymbol{\theta}^\star
\\[2pt]
\lambda_k-\lambda^\star
\end{pmatrix}
\right\|^2.
\]

In practice, SLSQP uses a BFGS approximation of the Lagrangian Hessian. Because
\[
\nabla^2_{\boldsymbol{\theta}\boldsymbol{\theta}}
L(\boldsymbol{\theta}^\star,\lambda^\star)
\]
is positive definite when \(\kappa<2\), the BFGS updates maintain positive definiteness and, under the Dennis--Mor\'e condition, yield superlinear convergence.
\end{proof}

The convergence guarantees above depend critically on the approximate Eikonal property (Assumption~\ref{asmp:approx_eikonal_supp}). We next provide a theoretical justification for the loss term used to encourage this property during training.

\subsection{Theoretical Justification of the Eikonal Loss}
\label{sec:eikonal_loss_property_supp}
To support our Eikonal regularization $\mathcal{L}_{\mathrm{grad}}$, we demonstrate that minimizing $\mathcal{L}_{\mathrm{grad}}$ can provide a volume bound of regions where Assumption~\ref{asmp:approx_eikonal_supp} fails. For analysis, we assume a uniform probability measure over $\Omega_B$. Recall that:
\begin{equation}
    \mathcal{L}_{\mathrm{grad}} = \mathbb{E}_{\boldsymbol{\theta}\sim\Omega_B}
\Big[
\big|\|\nabla g(\boldsymbol{\theta})\|-1\big|
\Big].
\end{equation}

\begin{proposition}[Volume Bound on Approximate Eikonal Failure]\label{prop:approx_eikonal_supp}Let $S_\delta$ denote the region where the approximate Eikonal condition fails for a given margin $\delta \in (0,1)$:$$S_\delta = \left\{ \boldsymbol{\theta} \in \Omega_B \;\middle|\; \big|\|\nabla_{\boldsymbol{\theta}} g(\boldsymbol{\theta})\| - 1\big| > \delta \right\}.$$If the Eikonal regularization satisfies $\mathcal{L}_{\mathrm{grad}} \le \epsilon$, then the probability measure of this failure region is bounded by:

\[\mathbb{P}(\boldsymbol{\theta} \in S_\delta) \le \frac{\epsilon}{\delta}.\]

\end{proposition}
\begin{proof}  

Define \(X(\boldsymbol{\theta}) = |\|\nabla g(\boldsymbol{\theta})\| - 1|\) on \(\Omega_B\) with the uniform measure.
Since \(\{X>\delta\}\subseteq\{X\ge \delta\}\), we have

\[
\mathbb{P}(X>\delta)\le \mathbb{P}(X\ge \delta).
\]

 Markov's inequality then gives
\[
     \mathbb{P}(X>\delta)\;\leq\; \mathbb{P}(X \geq \delta) \;\leq\; \frac{\mathbb{E}[X]}{\delta} \;=\; \frac{\mathcal{L}_{\mathrm{grad}}}{\delta} \;\leq\; \frac{\epsilon}{\delta}. 
\]
\end{proof}
\section{The Gap between Theory and Practice}

The convergence guarantees in \cref{sec:convergence_supp} rest on idealized assumptions that are only approximately satisfied in practice. We discuss the main discrepancies below and their implications for the implementation.

\paragraph{Smoothness (Assumption~\ref{asmp:smoothness_supp}).} The smoothness of an MLP is determined by its activation function. Softplus activations yield a $C^\infty$ network, while ReLU and ELU~\cite{clevert2015fast} are smooth almost everywhere. We observe no significant performance difference among these choices in practice.

\paragraph{Feasibility Consistency (Assumption~\ref{asmp:feasibility_supp}) Failure.} As indicated in \cref{sec:classifier}, the accuracy of collision indication is $93.9\%$ on the test set.

\paragraph{Approximate Eikonal (Assumption~\ref{asmp:approx_eikonal_supp}) Failure.} As stated in Proposition~\ref{prop:approx_eikonal_supp}, minimizing $\mathcal{L}_{\mathrm{grad}}$ bounds the volume of the region where Assumption~\ref{asmp:approx_eikonal_supp} fails, but exact satisfaction of this assumption is not guaranteed. \cref{fig:grad} reports the empirical distribution of $\|\nabla g\|$ on the test set: $95\%$ of samples satisfy the approximate Eikonal property with $\delta = 0.1$, suggesting the assumption holds for the vast majority of poses encountered in practice.

\begin{figure}[h]
\centering  
\includegraphics[width=0.7\columnwidth]{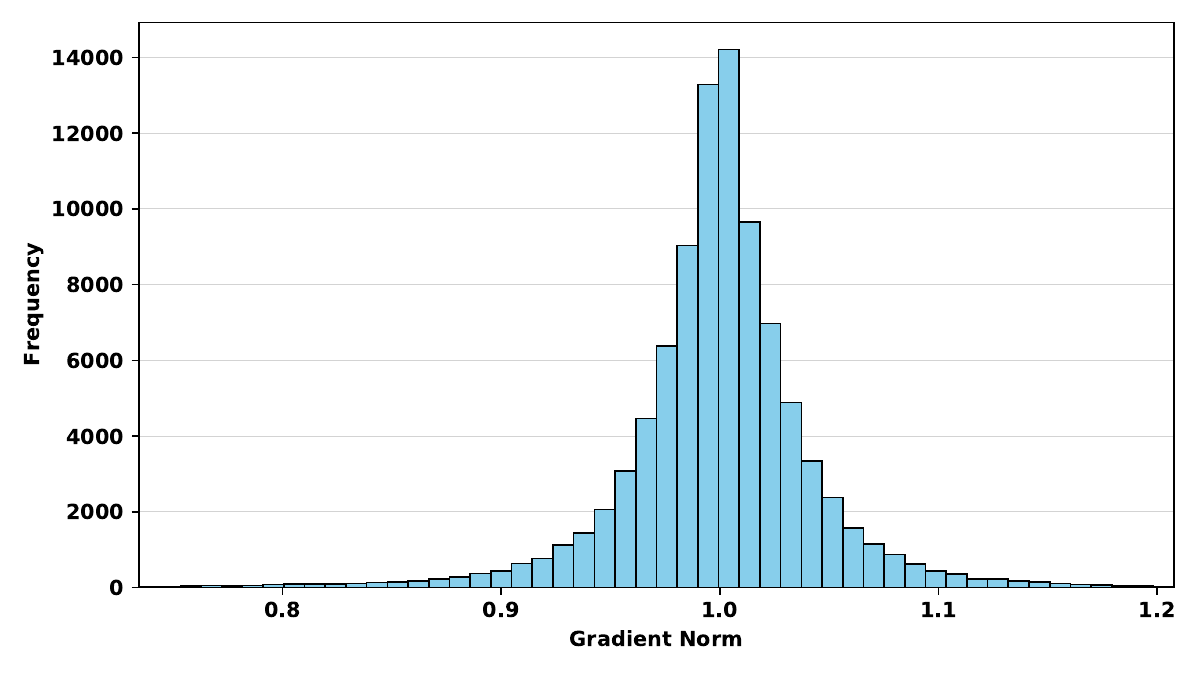}
 \caption{\textbf{Empirical verification of the approximate Eikonal property on the test set ($\approx 92k$ samples).} $0.5\%$ outliers on both sides are removed. $95\%$ of the samples satisfy the approximate Eikonal property with $\delta=0.1$.}
\label{fig:grad}
\end{figure}

\paragraph{Sampling.} The analysis in Proposition~\ref{prop:approx_eikonal_supp} assumes a uniform distribution over $\Omega_B$. In practice, samples are drawn from a data-induced distribution: we add Gaussian noise to existing motion data and project the perturbed poses to valid 6D rotations (\cref{sec:hwc}). This concentrates samples in regions of high practical relevance, though it does not exactly match the uniform measure assumed in the analysis.

\paragraph{Network Input Pre-processing.} Our training data consists solely of normalized 6D rotations (\cref{sec:hwc}), providing poor coverage of the full domain $\Omega_B$. However, unnormalized poses can arise during optimization. To keep inputs in-distribution, we apply Gram--Schmidt orthonormalization to all inputs to $g$, so the network effectively computes $\tilde{g}(\boldsymbol{\theta}) := g(\pi_{6D}(\boldsymbol{\theta}))$. This imposes the constraint that $\tilde{g}$ is constant on the level sets of $\pi_{6D}$:
\begin{equation}
    \tilde{g}(\boldsymbol{\theta}_1) = \tilde{g}(\boldsymbol{\theta}_2),
    \quad
    \forall\, \boldsymbol{\theta}_1, \boldsymbol{\theta}_2 \in \mathcal{D}
    \text{ s.t. }
    \pi_{6D}(\boldsymbol{\theta}_1) = \pi_{6D}(\boldsymbol{\theta}_2),
\end{equation}
restricting the effective input domain of $g$ to $SO(3)^J$ and potentially limiting its approximation capacity over the full $\Omega_B$.

\section{Limitations and Future Work}\label{sec:limit}
Currently, distances between poses and motions are measured solely using geometric metrics. In practical applications, however, users may care more about semantic fidelity. For example, whether a hand is exactly touching the head can be important in certain animations. For such applications, integrating our method with semantic distance metrics would be a valuable direction for future work. Our method can be seamlessly extended to parametric human models beyond SMPL, such as the Momentum Human Rig~\cite{ferguson2025mhr}. However, the current formulation assumes a fixed body shape. This assumption is sufficient for some applications. For example, in digital content creation, a character’s body shape is typically fixed, making it feasible to train the model once and then apply it to that character in diverse scenarios. Nevertheless, other applications may require the learned constraint function to generalize across a range of body shapes. Extending our method to handle varying body shapes is an important direction for future work.

\section{Our Model as a Classifier}
\label{sec:classifier}
In principle, the sign of $g$ indicates the collision status of a sample. Therefore, our method can also be used as a collision detector. We use the following metrics:
\begin{enumerate}
    \item Prediction accuracy (ACC). It measures whether the method can successfully predict the collision label.
    \item False negative rate (FNR). The rate at which a self-colliding mesh is predicted as collision-free.
\end{enumerate}
\begin{table}[h]
    \centering
    \caption{\textbf{Comparison of collision detection on our pose dataset.} 
    $\uparrow$ indicates higher values are better, 
    $\downarrow$ indicates lower values are better.}
    \label{table:cls}
    \resizebox{0.4\linewidth}{!}{
    \begin{tabular}{lcc}
    \toprule
    Method & ACC$\uparrow$ & FNR$\downarrow$  \\ 
    \midrule

    Classifier-baseline & 0.931& 0.035 \\
    \midrule
    Ours & 0.939 & 0.031 \\
    \bottomrule
    \end{tabular}
    }
\end{table}

The results are as shown in~\prettyref{table:cls}. Our method can serve as a classifier, achieving performance comparable to that of a standard binary classifier.

\section{Human Motion Collision Resolution: Implementation Details}
\label{sec:motion_losses}

In practice, to maintain the visual proximity of the optimized motion to the source motion, we define a motion distance term $d_{\text{motion}}$ that operates on the pose representation and the corresponding SMPL joints. Given a motion sequence
\(\mathbf{m} = [\boldsymbol{\theta}^{0}, \boldsymbol{\theta}^{1}, \cdots, \boldsymbol{\theta}^{T}]\)
and the source motion
\(\mathbf{m}_{s} = [\boldsymbol{\theta}^{0}_{s}, \boldsymbol{\theta}^{1}_{s}, \cdots, \boldsymbol{\theta}^{T}_{s}]\),
we preserve the proximity between $\mathbf{m}$ and $\mathbf{m}_s$ in the pose-parameter space using
\begin{align}
\mathcal{L}_{\text{feat}}
&= \frac{1}{T+1}\sum_{t=0}^{T}
    \|\boldsymbol{\theta}^{t} - \boldsymbol{\theta}^{t}_{s}\|_2^2.
\end{align}
To improve motion fidelity in 3D space, we convert each frame \(\boldsymbol{\theta}^{t}\)
into SMPL joint positions.
Specifically, let \(\mathbf{p}^{t} = \mathrm{SMPL}_{J}(\boldsymbol{\theta}^{t})\) and
\(\mathbf{p}^{t}_{s} = \mathrm{SMPL}_{J}(\boldsymbol{\theta}^{t}_{s})\),
where \(\mathrm{SMPL}_{J}(\cdot)\) outputs 3D joints via SMPL forward kinematics.
We supervise both joint configurations and their temporal changes:
\begin{align}
\mathcal{L}_{\text{pos}}
&= \frac{1}{T+1}\sum_{t=0}^{T}
    \|\mathbf{p}^{t} - \mathbf{p}^{t}_{s}\|_2^2 , \\[6pt]
\mathcal{L}_{\text{vel}}
&= \frac{1}{T}\sum_{t=0}^{T-1}
    \|(\mathbf{p}^{t+1} - \mathbf{p}^{t})
     - (\mathbf{p}^{t+1}_{s} - \mathbf{p}^{t}_{s})\|_2^2 .
\end{align}
$d_{\text{motion}}$ is defined as a weighted combination of the above losses:
\begin{equation}
d_{\text{motion}}(\mathbf{m}, \mathbf{m}_{s})
=
\mathcal{L}_{\text{feat}}
+ \lambda_{\text{joint}}\mathcal{L}_{\text{pos}}
+ \lambda_{\text{vel}}\mathcal{L}_{\text{vel}}.
\end{equation}
Here, $\mathcal{L}_{\text{feat}}$ encourages frame-wise similarity in the pose-parameter space,
while $\mathcal{L}_{\text{pos}}$ and $\mathcal{L}_{\text{vel}}$ preserve joint-level fidelity and temporal consistency in 3D. By default, we set $\lambda_{\text{joint}}$ to $1$ and $\lambda_{\text{vel}}$ to $0.1$.

\section{Details of Baseline Implementation}
\paragraph{VolumetricSMPL}
We follow the official implementation and adopt the hyperparameters provided in the paper.
Since only pretrained weights for SMPL-X are released, we map the joint rotations in the test set from SMPL to the corresponding SMPL-X joints and evaluate collisions under the SMPL-X model for a fair comparison.
Samples that exhibit self-collisions in SMPL but not in SMPL-X (140 out of 500 in the HwC benchmark set) are excluded from the evaluation.

\paragraph{COAP}
We follow the official implementation.
Self-collision resolution is implemented in \texttt{tutorials/untangle\_body.py}: body pose is optimized with SGD to minimize the learned self-penetration loss (COAP) plus a pose prior.
We use the hyperparameters provided in the repository (learning rate, pose prior weight, and self-penetration weight).
Optimization stops when the weighted self-penetration loss falls below the script's default threshold or when the maximum number of iterations (200) is reached.

\paragraph{Torch-mesh-isect}
\textit{Torch-mesh-isect}~\cite{tzionas2016capturing} can be found on GitHub. Collision handling is implemented through the file \texttt{examples/batch\_smpl\_untangle.py}. 
Notably, the original implementation does not include an internal stopping mechanism. 
To prevent indefinite execution, we set a maximum runtime of \texttt{3 minutes} for each case.

\paragraph{Classifier Baseline.}
This baseline follows the constrained optimization procedure in \cref{alg:cls_opt}.
We train a binary classifier on the HwC dataset to predict whether an SMPL pose is collision-free.
Let $cls(\boldsymbol{\theta};\phi)\in[0,1]$ denote the predicted probability that
$\iota(\mathcal{M}(\boldsymbol{\beta},\boldsymbol{\theta}))=+1$ under a fixed shape $\boldsymbol{\beta}$.
Given an initial self-colliding pose $\boldsymbol{\theta}_0$, we solve a pose-space constrained optimization problem
using the classifier output as a surrogate feasibility test.
We note that since N-Penetrate~\cite{tan2022n} is not open-sourced, we are not able to directly compare with that method.

\begin{algorithm}[ht]
\caption{\label{alg:cls_opt}Neural Collision Resolution with a Classifier}
\begin{algorithmic}[1]
\Require Initial pose $\boldsymbol{\theta}_0$, binary classifier $cls(\boldsymbol{\theta};\phi)$, constrained solver $\mathcal{F}$
\Ensure Optimized pose $\boldsymbol{\theta}^\star$
\State Initialize optimization with $\boldsymbol{\theta} \gets \boldsymbol{\theta}_0$
\State Define constraint set $\mathcal{C} = \{\, cls(\boldsymbol{\theta};\phi) > 0.5 \,\}$
\State Define objective function $d_{\mathrm{SMPL}}(\boldsymbol{\theta}, \boldsymbol{\theta}_0)$
\State Solve for
\[
\begin{aligned}
\boldsymbol{\theta}^\star = \mathcal{F}(
    &~\text{initial state}=\boldsymbol{\theta}_0, \\
    &~\text{constraints}=\mathcal{C}, \\
    &~\text{objective}=d_{\mathrm{SMPL}}(\boldsymbol{\theta}, \boldsymbol{\theta}_0))
\end{aligned}
\]
\State \Return $\boldsymbol{\theta}^\star$
\end{algorithmic}
\end{algorithm}

\section{Details of Active Learning}
\label{sec:active_learning}

\begin{algorithm}[ht]
\caption{\label{alg:ActiveLearning-a}Active Learning of \ours}
\begin{algorithmic}[1]
\State Sample an initial pose dataset $\mathcal{D}_{\theta}$
\State Train $g(\boldsymbol{\theta})$ using $\mathcal{L}_\text{\ours}$ on $\mathcal{D}_{\theta}$
\While{Not converged}
\State Sample an additional pose set $\mathcal{D}_{\theta}^{+}$
\For{each $\boldsymbol{\theta}_0 \in \mathcal{D}_{\theta}^{+}$}
\State Use $\boldsymbol{\theta}_0$ as the initial guess
\State Solve for $\boldsymbol{\theta}_\dagger,\,\mathcal{D}_{\theta}^{\dagger} \gets 
\argminH_{\boldsymbol{\theta}} \frac{1}{2}\big|g(\boldsymbol{\theta})\big|^2$
\State Augment dataset $\mathcal{D}_{\theta}^{+} \gets \mathcal{D}_{\theta}^{+}\cup \mathcal{D}_{\theta}^{\dagger}$
\EndFor
\State Retrain $g(\boldsymbol{\theta})$ using $\mathcal{L}_\text{\ours}$ on $\mathcal{D}_{\theta}^{+}$
\EndWhile
\end{algorithmic}
\end{algorithm}

In our standard setup, we construct $\mathcal{D}_{\theta}$ via random augmentation based on a seed set in the SMPL pose space. However, this approach often suffers from distribution bias, since the true underlying pose distribution is unknown. More critically, collision resolution requires $g(\boldsymbol{\theta})$ to accurately capture the \emph{decision boundary} of the exact collision indicator $\iota(\mathcal{M}(\boldsymbol{\beta},\boldsymbol{\theta}))$, i.e., the zero-level set $\{\boldsymbol{\theta}\mid g(\boldsymbol{\theta})=0\}$. In contrast, the precise shape of $g(\boldsymbol{\theta})$ far away from the boundary is less important, since those regions are mainly visited during intermediate steps of constrained optimization. Unfortunately, naive augmentation-based sampling does not emphasize this crucial near-boundary region.

To address these limitations, N-Penetrate~\cite{tan2022n} introduced an active-learning strategy that incrementally augments the training set. Following this idea, we let $\argminH$ denote an optimization procedure that returns not only the final solution but also all intermediate iterates encountered during optimization. At each active-learning iteration, we draw pose samples as usual. For each sampled pose $\boldsymbol{\theta}_0$, we solve:
\begin{align}
\label{eq:active-opt}
\argminH_{\boldsymbol{\theta}} \frac{1}{2}\big|g(\boldsymbol{\theta})\big|^2,
\end{align}
where $\argminH$ is used to collect all intermediate poses produced by the optimizer. The final converged solutions approximate the current zero-level set of $g(\boldsymbol{\theta})$, and the collected iterates concentrate samples near the boundary. This improves the accuracy of the learned decision boundary over iterations. The full active-learning pipeline is summarized in~\cref{alg:ActiveLearning-a}. We emphasize that this active-learning strategy is adopted from prior work~\cite{tan2022n} as an implementation detail, and we do not claim it as a contribution of this paper.

\section{More Human Motion Collision Resolution Results}
\label{sec:motion_collision}
\paragraph{Data.} We use $100$ sequences with the largest total penetration depths from the MotionFix dataset~\cite{athanasiou2024motionfix}.

\paragraph{Metrics.} We use the following metrics:
\begin{itemize}
\item Jitter~\cite{yi2022physical} evaluates the smoothness of the motion, measured in units of $10^2m/s^3$.
\item Foot Skating Ratio (FSR)~\cite{karunratanakul2023guided} measures the proportion of frames exhibiting foot skating artifacts. Since aggressive collision resolution can introduce unnatural motion patterns such as foot sliding, FSR serves as an indirect indicator of overall motion quality.
\item Residual Penetration Depth (RPD) measures the severity of residual interpenetration after optimization. It is computed as the frame-averaged penetration depth of the output motion.
\item Motion Feature Distance (MFD) measures the semantic discrepancy between the optimized and original motions in a learned motion feature space. Specifically, we extract motion features using a motion encoder~\cite{meng2025rethinking} and compute the feature-space distance between the optimized motion and its corresponding original motion.
\end{itemize}

\paragraph{Baselines.} We compare our method against two baselines.
\begin{itemize}
    \item Direct motion optimization. An alternative that uses the same optimization objective as ours, but optimizes the motion sequence $\mathbf{m}$ itself instead of the input noise $\mathbf{x}$ to the diffusion model $f$. The optimization process is:
    \begin{align*}
    \mathcal{L}_{\text{smooth}}(\mathbf{m}) &= \frac{1}{T}\sum_{t=0}^{T-1}
    \|\boldsymbol{\theta}^{t+1} - \boldsymbol{\theta}^{t}\|_2^2 . \\
        \mathbf{m}^\star &= \arg\min_{\mathbf{m}} \mathcal{Q}(\mathbf{m}) + \lambda_{\text{smooth}}\mathcal{L}_{\text{smooth}}(\mathbf{m})
    \end{align*}
    We add the additional smoothness term to improve temporal consistency. $\lambda_{\text{smooth}}$ is set to $0.5$.
    \item COAP (DNO)~\cite{mihajlovic2022coap}. We apply the same optimization algorithm DNO and $d_{\text{motion}}$ as our method while replacing the collision term with COAP self-collision loss. To keep the inference time comparable with our method, the number of sampled points is set to $50$ per body part. 
\end{itemize}

\paragraph{Results.}
\begin{table}[t]
\centering
\caption{Quantitative comparison on human motion collision resolution. Bold indicates the best result.}
\label{tab:motion_comparison}
{
\begin{tabular}{lcccc}
\toprule
Method &  Jitter $\downarrow$ & MFD $\downarrow$ & FSR (\%) $\downarrow$ & RPD $\downarrow$ \\
\midrule
GT &  0.3137 & 0.0000 & 5.08 & 1.8945 \\
\midrule
COAP (DNO)  & \textbf{0.5472} & 0.2948 & 3.52 & 0.1274 \\
Direct Opt.  & 2.5184 & \textbf{0.2322} & {5.82} & {0.0828} \\
Ours & {0.5684} & {0.2551} & \textbf{3.12} & \textbf{0.0738} \\
\bottomrule
\end{tabular}}
\end{table}

Quantitative results are reported in \cref{tab:motion_comparison}.
Our method provides the best balance between collision removal and motion quality. Direct optimization is the fastest and stays closest to the input motion in feature space, but leaves more residual collisions and introduces more jitter and foot skating. This comparison highlights the advantage of the generative prior in preserving more natural motion during collision resolution. COAP (DNO) improves smoothness over direct optimization, but still leaves more residual penetration and foot skating. Since our method and COAP (DNO) share the same optimization framework and motion objective, this gap suggests that our learned collision loss is more effective.

\bibliographystyle{splncs04}
\bibliography{main}
\end{document}